\documentclass{article}

     \PassOptionsToPackage{numbers, compress}{natbib}

     \usepackage[preprint]{neurips_2025}

\usepackage[utf8]{inputenc} 
\usepackage[T1]{fontenc}    
\usepackage{hyperref}       
\usepackage{url}            
\usepackage{booktabs}       
\usepackage{amsfonts}       
\usepackage{nicefrac}       
\usepackage{microtype}      
\usepackage{xcolor}         
\usepackage{tabularray}

\usepackage{amsmath}
\usepackage{amsthm}
\usepackage{graphicx}
\usepackage{subfigure}
\usepackage{enumitem}
\usepackage{bbm}
\usepackage{bm}
\usepackage{algpseudocode,algorithm,algorithmicx} 
\usepackage[title,toc,titletoc,page]{appendix}
\newcommand{\argmax}[1]{\underset{#1}{\operatorname{arg}\,\operatorname{max}}\;}

\newcommand{\argmin}[1]{\underset{#1}{\operatorname{arg}\,\operatorname{min}}\;}

\newtheorem{theorem}{Theorem}[section]
\newtheorem{proposition}[theorem]{Proposition}

\theoremstyle{definition}
\newtheorem{definition}[theorem]{Definition}
\newtheorem{assumption}[theorem]{Assumption}
\theoremstyle{remark}

\title{Bayesian Bandit Algorithms with Approximate Inference in Stochastic Linear Bandits}

\author{%
  $^1$Ziyi Huang, $^1$Henry Lam, $^{1,2}$Haofeng Zhang \thanks{Authors are listed alphabetically.} 
  \\
  $^1$Columbia University, New York, USA\\
  $^2$Morgan Stanley, New York, USA\\
  \texttt{zh2354@columbia.edu, henry.lam@columbia.edu, hz2553@columbia.edu} \\
}

\begin{document}

\maketitle

\begin{abstract}
Bayesian bandit algorithms with approximate Bayesian inference have been widely used in real-world applications. Despite the superior practical performance, their theoretical justification is less investigated in the literature, especially for contextual bandit problems. To fill this gap, we propose a theoretical framework to analyze the impact of approximate inference in stochastic linear bandits and conduct frequentist regret analysis on two Bayesian bandit algorithms, Linear Thompson Sampling (LinTS) and the extension of Bayesian Upper Confidence Bound, namely Linear Bayesian Upper Confidence Bound (LinBUCB). We demonstrate that when applied in approximate inference settings, LinTS and LinBUCB can universally preserve their original rates of regret upper bound but with a sacrifice of larger constant terms. These results hold for general Bayesian inference approaches, assuming the inference error measured by two different $\alpha$-divergences is bounded. Additionally, by introducing a new definition of well-behaved distributions, we show that LinBUCB expedites the regret rate of LinTS from $\tilde{O}(d^{3/2}\sqrt{T})$ to $\tilde{O}(d\sqrt{T})$, matching the minimax optimal rate. To our knowledge, this work provides the first regret bounds in the setting of stochastic linear bandits with bounded approximate inference errors. 
\end{abstract}

\section{Introduction}
\label{contextual}
Stochastic bandit problems and their various generalizations \citep{robbins1952some,lattimore2020bandit,bubeck2012regret} are essential sequential decision-making problems that aim to find optimal adaptive strategies to maximize cumulative reward. The main challenge in these problems is to balance the tradeoff between exploration and exploitation, as one can only observe the outcome of the selected action, whereas counterfactual outcomes for alternative actions are left unknown. By leveraging posterior samples of the model, Thompson sampling \citep{thompson1933likelihood,chapelle2011empirical} and its diverse extensions \citep{kaufmann2012bayesian,kaufmann2018bayesian,srinivas2009gaussian,lu2017ensemble,qin2022analysis,zhang2021neural,hwang2023combinatorial,chen2023pseudo} provide elegant Bayesian solutions to automatically balance exploration and exploitation based on historical observations. However, conventional Thompson sampling requires \textit{exact} Bayesian inference, i.e., full access to sampling from exact posterior distributions. This makes it less applicable to complex models, such as deep neural networks, in which drawing samples from exact posterior distributions tends to be intractable or computationally demanding. A practical enhancement is to integrate approximate Bayesian inference methods into Thompson sampling frameworks to enable \textit{approximate} posterior sampling \citep{riquelme2018deep,snoek2015scalable,osband2016deep,urteaga2018variational,guo2020deep,osband2023approximate,duran2022efficient,zhu2023scalable}. 

Existing theoretical results are mainly derived around Thompson sampling under \textit{exact} Bayesian inference \citep{agrawal2012analysis,agrawal2013further,kaufmann2012thompson,russo2014learning,agrawal2013thompson,abeille2017linear}, whereas studies of \textit{approximate} Bayesian inference are less explored. Generally, it is more challenging to analyze Thompson sampling with approximate Bayesian inference, in comparison with exact algorithms. In approximate Bayesian inference, the inaccessibility of exact posterior sampling can introduce another level of discrepancy and thus leads to differences in resulting theories and solutions. 
Phan et al. \cite{phan2019thompson} show that Thompson sampling with an $\alpha$-divergence-bounded inference error could incur a linear regret in multi-armed bandit problems. Their study indicates a frustrating phenomenon that approximate Bayesian bandit methods could fail theoretically in a worst-case scenario, despite their superior empirical performance. To mitigate this issue, Mazumdar et al. \cite{mazumdar2020approximate} (and subsequently \cite{karbasi2023langevin}) integrate a specific approximate Bayesian inference approach, the Markov chain Monte Carlo Langevin algorithm, into multi-armed Thompson sampling to achieve the optimal regret rate. Along with increasing sample sizes, their approach allows the inference error to vanish, as opposed to being bounded. On the other side, Huang et al. \cite{huang2023optimal} adopt a similar setting to \cite{phan2019thompson} to investigate a general approximate Bayesian inference approach. By strengthening the assumption in \cite{phan2019thompson}, their study shows that optimal regret rate is achievable when the inference error is bounded by two $\alpha$-divergences. However, all of the above studies are developed for multi-armed bandits, whose setting is less ideal in practical applications where contextual information is involved. 


To bridge the gap, our study investigates contextual bandit algorithms \textit{in the presence of approximate inference}, which are widely used in various real-world applications, such as personalized healthcare platforms \citep{tewari2017ads} and recommendation systems \citep{li2010contextual}. As a generalization of multi-armed bandit problems, contextual bandit problems allow features, also known as ``context'', to impact the reward. 
We consider stochastic linear contextual bandit problems, where the expected reward is a linear function of the context. Linear Thompson Sampling (LinTS), introduced by \cite{agrawal2013thompson}, is a well-established framework that can address the problems. Building upon this foundation, Xu et al. \cite{xu2022langevin} propose integrating LinTS with a specific Bayesian inference method, the Langevin algorithm.
Their study allows the approximate posterior distribution to be sufficiently close to the exact posterior distribution, leading to a vanishing inference error, which extends a similar study in multi-armed bandit \citep{mazumdar2020approximate}. However, in other inference approaches, such as variational inference \citep{blei2017variational}, the approximate posteriors might incur a systematic computational bias, resulting in a bounded, rather than vanishing inference error. As such, there is a need to develop theoretical frameworks for analyzing \textit{general} Bayesian inference approaches in contextual bandits. 

In this study, we develop a theoretical framework to analyze Bayesian bandit algorithms with general inference approaches in linear contextual bandits. 
We show that two Bayesian inference approaches, LinTS and Linear Bayesian Upper Confidence Bound (LinBUCB), with approximate inference, can preserve their original rates of the regret upper bound as the ones without approximate inference.
Since LinTS incurs a regret bound of $\tilde{O}(d^{3/2}\sqrt{T})$, which has an undesirable $d^{1/2}$ gap to the minimax optimal rate $O(d\sqrt{T})$ \citep{hamidi2020frequentist}, this motivates us to develop a Bayesian bandit approach that could achieve the minimax optimal rate. To this end, we extend another Bayesian bandit algorithm, Bayesian Upper Confidence Bound (BUCB) that was originally designed for multi-armed bandits \citep{kaufmann2012bayesian,kaufmann2018bayesian,huang2023optimal}, to the linear contextual bandits, and term it LinBUCB.
We establish the regret upper bound of LinBUCB in two settings, with and without approximate inference. By introducing a new type of well-behaved distributions, we show that the regret bound of LinBUCB with exact inference can be boosted to match the minimax optimal rate $O(d\sqrt{T})$, eliminating the gap in LinTS. We further investigate the performance of LinBUCB with approximate inference to demonstrate the effectiveness and generality of our framework. Our study uncovers the phenomenon of regret rate maintenance in Bayesian bandit approaches with approximate inference, showing promise to improve computational efficiency in practical applications. Lastly, our framework provides flexibility in posterior distribution selections that extend beyond traditional assumptions, making it more widely applicable.
Our contributions are summarized as follows:\\
1) We develop the first theoretical framework to analyze the impact of bounded approximate inference error on Bayesian bandit algorithms in linear contextual bandits. In particular, we develop multiple fundamental tools to analyze the sensitivity of posterior distributions with respect to inference error. \\
2) We show LinTS with approximate inference achieves a $\tilde{O}(d^{3/2}\sqrt{T})$ regret upper bound when the inference error measured by two $\alpha$-divergences is bounded, preserving the same regret rate as LinTS without approximate inference. We also demonstrate that one bounded $\alpha$-divergence alone is insufficient to guarantee a sub-linear regret.\\
3) We extend our analysis to another Bayesian bandit algorithm, LinBUCB, and establish regret upper bounds of LinBUCB in both approximate and exact inference settings, showing that the same regret rates $\tilde{O}(d^{3/2}\sqrt{T})$ are sustained. Moreover, by proposing a new type of well-behaved distributions, we further expedite the regret rate of LinBUCB to match the minimax optimal rate $O(d\sqrt{T})$.

This paper considers the \textit{frequentist regret} setting in which LinTS is known to have a suboptimal regret rate \citep{agrawal2013thompson,hamidi2020frequentist}. This is distinguished from the \textit{Bayesian regret} setting where LinTS has been shown to achieve the optimal Bayesian regret rate \citep{russo2014learning,atsidakou2023finite}. Additionally, alternative randomization mechanisms beyond the Bayesian framework have also been proposed \citep{kveton2019garbage,kveton2020randomized,osband2015bootstrapped}, offering different types of exploration in bandit algorithms.

\section{Methodology} \label{sec:meth}

\textbf{Notations.} We let $\|\cdot\|$ denote the $L^2$-norm in $\mathbb{R}^d$ and $x^\top$ the transpose
of $x\in \mathbb{R}^d$. For a positive definite matrix $M \in \mathbb{R}^{d\times d}$, we use $\|x\|_{M}$ to denote the weighted $L^2$-norm $\|x\|_{M}=\sqrt{x^\top M x}$ for any $x\in \mathbb{R}^d$. For a random variable or a univariate distribution $\rho$, we use $\mathcal{Q}(\gamma, \rho)$ to denote the quantile function associated with $\rho$, i.e., $\mathcal{Q}(\gamma, \rho) = \inf\{x\in\mathbb{R}: \gamma \le F_\rho(x)\}$, where $F_\rho$ is the cumulative distribution function of $\rho$. The $\tilde{O}(\cdot)$ notation is the $O(\cdot)$ notation with suppressed logarithmic terms \citep{cormen2022introduction}.

We consider stochastic linear bandit problems with the following problem settings. At each time step $t=1,...,T$, where $T$ is the time horizon, the learner is given an arbitrary set of arms (also known as actions, or contexts) $\mathcal{X}_t \subset \mathbb{R}^d$. Note that $\mathcal{X}_t$ may change as $t$ changes and could be a finite or infinite set.
Then the learner pulls an arm $x_t \in \mathcal{X}_t$ based on the past observations and certain strategies, and observes a reward generated as $r_{t+1} = r_{t+1}(x_t) = x_t^\top \theta^* + \xi_{t+1}$, where $\theta^*\in \mathbb{R}^d$ is a fixed but unknown parameter, and $\xi_{t+1}$ is a zero-mean noise. As the ground truth $\theta^*$ is unknown, we define the following two notations for a general $\theta$:
$$x_t^*(\theta)= \argmax{x\in\mathcal{X}_t} x^\top \theta, \ J_t(\theta)= \sup_{x\in\mathcal{X}_t} x^\top \theta.$$
The optimal arm at time $t$ is the one with the maximum expected reward, i.e., $x_t^*(\theta^*)$, which has the reward $J_t(\theta^*)$. We write $x_t^*=x_t^*(\theta^*)$ for simplicity.
Suppose that at time step $t$, the learner selects an arm $x_t \in \mathcal{X}_t$ according to a certain strategy, which incurs a \textit{regret} equal to the difference in expected reward between the
optimal arm $x_t^*(\theta^*)$ and the arm $x_t$:
$J_t(\theta^*)-x_t^\top \theta^*.$
The objective of the learner is to devise a strategy to minimize the cumulative frequentist regret (or equivalently, maximize the accumulated frequentist reward) up to step $T$, i.e., 
$$R(T) = \sum_{t=1}^T (J_t(\theta^*)-x_t^\top \theta^*).
$$
In general, a wise strategy should be sequential, in the sense that the upcoming actions are adjusted by past observations from historical interactions.
Suppose that $\mathcal{F}^{x}_t = \sigma(\mathcal{F}_0, \sigma(x_1, r_2,..., r_t, x_t))$ is the $\sigma$-algebra generated by all the information observed up to time $t$ (including the action to take), where $\mathcal{F}_0$ contains any prior knowledge. 



In the reward function, the true parameter $\theta^*$ is fixed and unknown to the learner. One approach is to use regularized least-squares (RLS) to estimate its value:
$\argmin{\theta} \frac{1}{n}\sum_{s=1}^n (x_s^\top \theta-r_{s+1})^2+\lambda \|\theta\|^2,$
where $(x_1, ..., x_t) \in \Pi^t_{i=1} \mathcal{X}_i$ is the sequence of historical arms pulled, $(r_2, ..., r_{t+1})$ is the corresponding rewards, and $\lambda \in \mathbb{R} ^+$ is a regularization parameter.
The design matrix $V_t$ and the RLS estimate $\hat{\theta}_t$ are given as:
\begin{equation}\label{RLS_estimation}
V_t = \lambda I + \sum_{s=1}^{t-1} x_s x_s^\top, \ \hat{\theta}_t = V_t^{-1}\sum_{s=1}^{t-1}x_s r_{s+1}.
\end{equation}
We first recall important properties for RLS estimates. As in previous work \citep{abbasi2011improved,agrawal2013further,abeille2017linear,flynn2024improved}, we introduce the following standard assumptions.

\begin{assumption} \label{linearassu}
For any $t \in [T]$, the arm set $\mathcal{X}_t$ is a bounded closed (compact) subset of $\mathbb{R}^d$ such that $\|x\| \le 1$ for all $x \in \mathcal{X}_t$. There exists a known constant $S \in \mathbb{R}^+$ such that $\|\theta^*\| \le S$. Moreover, the noise process $\{\xi_t\}_t$ is a martingale difference sequence given $\mathcal{F}_t^x$ and it is conditionally $\nu$-sub-Gaussian for some constant $\nu \ge 0$, i.e., $\forall \alpha \in \mathbb{R},\ \forall t \geq 1,$
$\mathbb{E} [\xi _{t+1}| \mathcal{F}_t^x] = 0, \ \mathbb{E}[e^{\alpha \xi _{t+1}}| \mathcal{F}_t^x] \leq \exp (\alpha^2 \nu^2/2).$
\end{assumption}

\begin{proposition}[\citep{abbasi2011improved,abeille2017linear}]\label{prop_abba}
Suppose Assumption \ref{linearassu} holds. Consider the $\mathcal{F} _t^x$ - adapted sequence $(x_1, ..., x_t)$ and the RLS estimator $\hat{\theta} _t$ defined above. For any $\delta \in (0,1)$, with probability at least $1-\delta$ (with respect to the noise $\{\xi_t\}_t$ and any source of randomization in the choice of the arms),
\begin{equation}\label{abba_eq}
\|\hat{\theta}_t - \theta^*\|_{V_t} \leq \beta_t(\delta), \ |x^\top (\hat{\theta}_t - \theta^*)| \leq \|x\|_{V_t^{-1}} \beta_t(\delta),
\end{equation}
for any $t \geq 1$ and any $x \in \mathbb{R}^d$, where
\begin{equation}
  \beta_t(\delta) = \nu \sqrt{2 \log \frac{(\lambda + t)^{d/2} \lambda^{-d/2}}{\delta}} + \sqrt{\lambda}S.   
\end{equation}

Moreover, for any arbitrary sequence $(x_1, ..., x_t) \in \Pi^t_{i=1} \mathcal{X}_i$,  let $V_{t+1}$ be the corresponding design matrix (Equation \eqref{RLS_estimation}), and then we have
\begin{equation} \label{abba_eq2}
\sum_{s=1}^t \|x_s\|_{V_s^{-1} }^2 \leq 2 \log \frac{\det (V_{t+1})}{ \det (\lambda I) } \leq 2d \log (1 + \frac{t}{\lambda}).
\end{equation}
\end{proposition}
The above theorem indicates that the ellipsoid $\mathcal{E}_t^{RLS} = \{ \theta \in \mathbb{R}^d | \ \|\theta - \hat{\theta}_{t}\|_{V_t} \leq \beta_t(\delta')\}$ is centered around $\hat{\theta}_t$, where $\delta' = \frac{\delta}{4T}$ is a valid confidence region for $\theta^*$ uniformly for any time step $t$. Since $\hat{\theta}_t$ in Equation \eqref{RLS_estimation} is a point estimation, additional uncertainty information is required to tackle the exploration and exploitation tradeoff. In Sections \ref{sec:LinTS} and \ref{sec:LinBUCB}, we introduce two approaches leveraging Bayesian uncertainty estimation.

\subsection{LinTS with Approximate Inference} \label{sec:LinTS}

The general idea of Thompson sampling is to maintain a distribution on $\theta$ that is updated from past observations to balance exploration and exploitation. Agrawal \& Goyal \cite{agrawal2013thompson} define LinTS as a Bayesian linear bandit algorithm where a Gaussian posterior over $\theta^*$ is updated according to the observed rewards. Particularly, if $\tilde \theta_t \sim \mathcal{N}(\hat{\theta}_t, \beta^2 V_t^{-1})$ and the likelihood function is $r_{t+1}\sim \mathcal{N}(x_t^\top \theta^*, \beta^2)$, then the posterior would be $\tilde \theta_{t+1} \sim \mathcal{N}(\hat{\theta}_{t+1}, \beta^2 V_{t+1}^{-1})$. Note that this likelihood function is fictitious only for the algorithm design, while in Assumption \ref{linearassu}, ground-truth $r_{t+1}-x_t^\top \theta^*$ can be any sub-Gaussian. In the subsequent study,  Abeille \& Lazaric \cite{abeille2017linear} relieve the need for the likelihood function by directly defining the right posterior distribution that the samples should be drawn from. 
At any time step $t$, given the RLS estimate $\hat{\theta}_t$ and the design matrix $V_t$, LinTS draws a posterior sample $\tilde{\theta}_t$ such that $\beta_t(\delta')^{-1} V_t^{1/2}(\tilde{\theta}_t-\hat{\theta}_t)$ follows from a time-independent ``well-behaved'' distribution (defined in Assumption \ref{def_D1} below). Specifically, this requirement confers an appropriate exploration capability to the posterior sample $\tilde{\theta}_t$, allowing LinTS to perform the right amount of exploration (anti-concentration) and exploitation (concentration). 
Then based on 
$\tilde \theta_t$, the learner selects an optimal arm $x_t = x_t^* (\tilde \theta_t)$, observes a reward $r_{t+1}$, and updates $V_t$ and $\hat{\theta}_t$ according to Equation \eqref{RLS_estimation}. 
\begin{assumption}[\citep{abeille2017linear}] \label{def_D1}
$\mathcal{D}^{1}$ is a multivariate distribution on $\mathbb{R}^d$ absolutely continuous with respect to the Lebesgue measure that satisfies the following properties:\\
1. (Anti-concentration) There exists a strictly positive $\kappa_1>0$ such that for any $u \in \mathbb{R}^d $ with $\|u\| = 1$,
\begin{equation}
    \mathbb{P}_{\eta \sim \mathcal{D}^{1}}(u^\top \eta \geq 1) \geq \kappa_1,
\end{equation}
2. (Concentration Type-I) There exist two positive constants $c_1$ and $c_1'$ such that for any $\delta \in (0, 1)$, 
\begin{equation}
    \mathbb{P}_{\eta \sim \mathcal{D}^{1}}\left(\|\eta\| \leq \sqrt{c_1d\log \frac{c_1'd}{\delta} }\right)\geq 1 -\delta.
\end{equation}
\end{assumption}

We rename the ``Concentration'' in \cite{abeille2017linear} as ``Concentration Type-I'' in Assumption \ref{def_D1}, since we will introduce another type of ``Concentration'' (called ``Concentration Type-II'') in Section \ref{sec:LinBUCB}. 

\begin{proposition}[\citep{abeille2017linear}]
The standard Gaussian distribution $\mathcal{N}(0, I_{d})$ satisfies Assumption \ref{def_D1}.   
\end{proposition}



We use $\Pi_t$ to denote the exact posterior distribution defined in \cite{abeille2017linear}. In practice, directly sampling from $\Pi_t$ might be arduous, potentially because the posterior is intractable or computationally demanding. In this case, applying an efficient method for approximate Bayesian inference can give rise to a sample from the approximate distribution $Q_t$. We call this method \textit{LinTS with approximate inference}. The pseudo-code of LinTS w/o Approximate Inference is presented in Algorithm \ref{algo:LinTS} in Appendix \ref{code_lints}.

\subsection{LinBUCB with Approximate Inference} \label{sec:LinBUCB}

Even though the minimax optimal regret rate for linear contextual bandit algorithms is $O(d\sqrt{T})$ \citep{dani2008stochastic}, LinTS could only achieve a regret bound of $\tilde{O}(d^{3/2}\sqrt{T})$ \citep{agrawal2013thompson} with an undesirable $d^{1/2}$ gap to the minimax optimality. As pointed out by \cite{hamidi2020frequentist}, this gap does not originate from unrefined theoretical derivation but from the ``conservative over-exploration'' nature. This motivates us to investigate another Bayesian bandit approach, the Bayesian Upper Confidence Bound (BUCB) \citep{kaufmann2012bayesian}, which shares a similar spirit as Thompson sampling but leverages the posterior quantiles rather than posterior samples. In this section, we consider a contextual extension of BUCB into the linear contextual bandit setting, termed Linear Bayesian Upper Confidence Bound (LinBUCB). Its pseudo-code is presented in Algorithm \ref{algo:LinBUCB} in Appendix \ref{code_linbucb}. Although similar ideas have been empirically implemented in practical studies \citep{guo2020deep,riquelme2018deep}, its theoretical guarantee is not fully established in linear contextual bandits. In Section \ref{sec:regretLinBUCB}, we perform a thorough regret analysis of LinBUCB. 

Like LinTS, LinBUCB maintains a posterior distribution on $\theta$, leading to a posterior distribution on $x^\top \theta$ for any available arm $x\in\mathcal{X}_t$. Their major differences lie in the indices of the posterior distribution. For each arm $x$, LinTS compares the posterior sample $x^\top \tilde{\theta}$ based on one instantiation $\tilde{\theta}$, whereas LinBUCB computes a quantile of the posterior distribution $\mathcal{Q}(\gamma, x^\top \theta)$, where $\mathcal{Q}(\gamma, \rho)$ is denoted as the quantile function of the random variable $\rho$. Lastly, the arm selected by LinBUCB is the one with the highest posterior quantile. 

In Assumption \ref{def_D2}, we further introduce a new class of distributions that possess the same anti-concentration property as in Assumption \ref{def_D1} but a slightly different concentration property. Note that the standard Gaussian distribution satisfies Assumption \ref{def_D2}. With this new assumption, we will show in Section \ref{sec:regretLinBUCB} that LinBUCB can facilitate the regret rate of LinTS from $\tilde{O}(d^{3/2}\sqrt{T})$ to $\tilde{O}(d\sqrt{T})$ that matches the minimax optimal rate up to logarithmic factors.

\begin{assumption} \label{def_D2}
$\mathcal{D}^{2}$ is a multivariate distribution on $\mathbb{R}^d$ absolutely continuous with respect to the Lebesgue measure that satisfies the following properties:

1. (Anti-concentration) There exists a strictly positive $\kappa_1>0$ such that for any $u \in \mathbb{R}^d $ with $\|u\| = 1$,
\begin{equation}
    \mathbb{P}_{\eta \sim \mathcal{D}^{2}}(u^\top \eta \geq 1) \geq \kappa_1.
\end{equation}
2. (Concentration Type-II) For any $\delta \in (0,1)$, there exists a constant $\hat{c}_1(\delta)>0$ (which is free of $d$, only depending on $\delta$) such that for any $u \in \mathbb{R}^d $ with $\|u\| = 1$,
\begin{equation}
    \mathbb{P}_{\eta \sim \mathcal{D}^{2}}(u^\top \eta \leq \hat{c}_1(\delta)) \geq 1-\delta.
\end{equation}
\end{assumption}

\begin{proposition} \label{prop_D2}
The standard Gaussian distribution $\mathcal{N}(0, I_{d})$ satisfies Assumption \ref{def_D2}.   
\end{proposition}

\section{Regret Analysis}  \label{sec:theo}

In this section, we perform theoretical analysis on LinTS and LinBUCB to establish their finite-time frequentist regret bounds across multiple settings. Section \ref{sec:kldivergence} provides the necessary background of $\alpha$-divergence, which serves as our primary measurement for the Bayesian inference error in Algorithms \ref{algo:LinTS} and \ref{algo:LinBUCB}. Section \ref{sec:regretLinTS} analyzes the finite-time regret of LinTS with approximate inference. 
Section \ref{sec:regretLinBUCB} analyzes the finite-time regret of LinBUCB, in both approximate and exact inference settings. Section \ref{sec:negativeresults} presents negative results for LinTS and LinBUCB when only one bounded $\alpha$-divergence is assumed.

\subsection{The Alpha Divergence} \label{sec:kldivergence}

As a generalization of the Kullback–Leibler (KL) divergence, the $\alpha$-divergence is a common way to measure errors in inference methods, which is widely used in variational inference approaches \citep{blei2017variational,kingma2013auto,li2016renyi,daudel2023alpha}. Previous studies define the $\alpha$ parameter in different ways. In our study, we adopt the generalized version of \textit{Tsallis's} $\alpha$-divergence \citep{zhu1995information,minka2005divergence,cichocki2010families,phan2019thompson,huang2023optimal}, which is slightly different from \textit{Renyi's} definition of $\alpha$-divergence \citep{renyi1961measures}. 
\begin{definition}
The $\alpha$-divergence between two univariate or multivariate distributions $P_1$ and $P_2$ with density functions $p_1(x)$ and $p_2(x)$ with respect to the Lebesgue measure is defined as: 
$D_{\alpha}(P_1,P_2)=\frac{1}{\alpha(\alpha-1)}\left(\int p_1(x)^{\alpha}p_2(x)^{1-\alpha}dx-1\right),$
where $\alpha\in \mathbb{R}$ and the case of $\alpha=0$ and $1$ is defined as the limit. We define
$D_{\alpha}(X_1,X_2)=D_{\alpha}(P_1,P_2),$
if the two random vectors $X_1$ and $X_2$ follow the distributions $P_1$ and $P_2$, respectively. According to the definition, $D_\alpha(P_1, P_2)\ge 0$ and $D_\alpha(P_1, P_2)=D_{1-\alpha}(P_2, P_1)$ hold for any $\alpha\in\mathbb{R}$.    
\end{definition}

As in previous studies \citep{lu2017ensemble,phan2019thompson,huang2023optimal}, we use $\alpha$-divergence to measure the inference error, i.e., the statistical distance between $\Pi_t$ and $Q_t$ in Algorithms \ref{algo:LinTS} and \ref{algo:LinBUCB}, with the following assumption.
\begin{assumption} \label{assu0}
$Q_{t}$ is a multivariate distribution on $\mathbb{R}^d$ absolutely continuous with respect to the Lebesgue measure, whose support is the same as $\Pi_{t}$'s support. There exists a positive value $\epsilon \in (0,+\infty)$ and two different parameters $\alpha_1>1$ and $\alpha_2<0$ such that
\begin{equation} \label{equ:error}
D_{\alpha_1}(\Pi_{t}, Q_{t})\le \epsilon, \ 
D_{\alpha_2}(\Pi_{t}, Q_{t})\le \epsilon, \ \forall t\in [T].  
\end{equation}
\end{assumption}

Note the symmetry: $D_{\alpha_1}(\Pi_{t}, Q_{t})=D_{1-\alpha_1}(Q_{t},\Pi_{t})$ and $D_{\alpha_2}(\Pi_{t}, Q_{t})=D_{1-\alpha_2}(Q_{t},\Pi_{t})$, so Assumption \ref{assu0} can be equivalently expressed by using $D_{\alpha}(Q_{t},\Pi_{t})$. Previous studies \citep{phan2019thompson,huang2023optimal} show that one bounded $\alpha$-divergence is insufficient to guarantee a sublinear regret. Thus, it is essential to assume two, rather than one, bounded $\alpha$-divergences, to stipulate that the approximate distribution is close to the posterior from two different ``directions''. In the following, we show that the rate of the regret upper bound can be preserved under Assumption \ref{assu0}. Note that our subsequent results are not tied to any specific approximate inference methods; they hold for any approximate distributions satisfying Assumption \ref{assu0}. Moreover, Assumption \ref{assu0} does not require the threshold $\epsilon$ to be small; instead, $\epsilon$ can be any finite positive number. Lastly, when one $\alpha$-divergence is small, it does not necessarily imply that any other $\alpha$-divergences are large or infinite.

\subsection{Finite-Time Regret Bound of LinTS} 
\label{sec:regretLinTS}

In this section, we derive an upper bound for the finite-time frequentist regret for LinTS with approximate inference in Algorithm \ref{algo:LinTS}, assuming that the inference error satisfies Assumption \ref{assu0}. 



\begin{theorem}[Regret of LinTS] \label{thm:LinTS}
Suppose that Assumption \ref{linearassu} holds. $\Pi_t$ is derived from Assumption \ref{def_D1} and $Q_t$ satisfies Assumption \ref{assu0}. Then the regret of LinTS with approximate inference is bounded by
$$R(T) \le \left( \beta_T (\delta') + \hat{\gamma}_T (\delta')(1 + \frac{4}{\kappa_2})\right)\sqrt{2Td \log (1+\frac{T}{\lambda})} $$ $$ + \frac{4\hat{\gamma}_T (\delta')}{\kappa_2}
\sqrt{\frac{8T}{\lambda}\log \frac{4}{\delta}},
$$
with probability $1-\delta$ ($\delta'=\frac{\delta}{4T}$). Here,
$\hat{\gamma}_t (\delta) = \beta_t(\delta) \sqrt{c_2 d \log \frac{c_2'd}{\delta}}$, $\kappa_2=\left(\epsilon\alpha_1 (\alpha_1-1)+1\right)^{\frac{1}{1-\alpha_1}} \kappa_1^{\frac{\alpha_1}{\alpha_1-1}}$, 
$c_2 = c_1 + \frac{\alpha_2-1}{\alpha_2}$, and $c_2'=\frac{c_1'}{\left(\epsilon\alpha_2(\alpha_2-1)+1\right)^{\alpha_2}}$.
\end{theorem}

Theorem \ref{thm:LinTS} exhibits explicitly the impact of approximate inference error on LinTS. It indicates that the regret of LinTS with approximate inference is bounded above by $\tilde{O}(d^{3/2}\sqrt{T})$. This preserves the same rate of the regret upper bound as in LinTS without approximate inference \citep{abeille2017linear}, with only constant terms deteriorating: $\frac{1}{\kappa_2}> \frac{1}{\kappa_1}$, $c_2 > c_1$, and $c_2'>c_1'$. This is reasonable since the use of approximate distributions, rather than exact distributions, introduces additional error. In addition, the regret upper bound in Theorem \ref{thm:LinTS} holds for any $\epsilon$-error threshold on the two $\alpha$-divergences, as $\Pi_{t}$ becomes increasingly concentrated around the ground-truth parameter $\theta^*$ along with increasing $t$, and a finite value of $\epsilon$ can force $Q_{t}$ to ultimately inherit this property. Moreover, Theorem \ref{thm:LinTS} establishes a universal upper bound that applies to any inference method satisfying Assumption \ref{assu0}. Consequently, it does not imply that LinTS with approximate inference is universally worse than LinTS with exact inference. Instead, performance comparisons between the two can vary, depending on the specific problem instance and the choice of inference method.

\subsection{Finite-Time Regret Bound of LinBUCB}  \label{sec:regretLinBUCB}

In this section, we first derive the upper bounds of the finite-time frequentist regret for standard LinBUCB without approximate inference under two different assumptions, Assumption \ref{def_D1} or Assumption \ref{def_D2}. Then we generalize the result to LinBUCB with approximate inference by performing a similar analysis as in Section \ref{sec:regretLinTS}.

\begin{theorem}[Regret of LinBUCB \textit{without} approximate inference] \label{thm:LinBUCB1}
Suppose that Assumption \ref{linearassu} holds.

1) Suppose $\Pi_t$ is derived from Assumption \ref{def_D1}. Set $1-\kappa_1\le \gamma <1$ in Algorithm \ref{algo:LinBUCB}. Then with probability $1-\delta$, the regret of LinBUCB without approximate inference is bounded by
$$R(T) \le \beta_T (\delta) \left(\sqrt{c_1d\log \frac{c_1'd}{1-\gamma} } + 1\right) \sqrt{2Td \log (1+\frac{T}{\lambda})}.
$$

2) Suppose $\Pi_t$ is derived from Assumption \ref{def_D2}. Set $1-\kappa_1\le \gamma <1$ in Algorithm \ref{algo:LinBUCB}. Then with probability $1-\delta$, the regret of LinBUCB without approximate inference is bounded by
$$R(T) \le \beta_T (\delta) \left(\hat{c}_1(1-\gamma) + 1\right) \sqrt{2Td \log (1+\frac{T}{\lambda})}.
$$
\end{theorem}

In Theorem \ref{thm:LinBUCB1}, part 1) shows that under the same distribution assumption as in LinTS (Assumption \ref{def_D1}), LinBUCB achieves the same rate of $\tilde{O}(d^{3/2}\sqrt{T})$ as in LinTS, whereas part 2) indicates that LinBUCB can achieve a faster rate $\tilde{O}(d\sqrt{T})$ that matches the minimax optimal rate  $O(d\sqrt{T})$ up to logarithmic factors when Assumption \ref{def_D2} holds. Part 2) reveals the advantages of the new concentration property in Assumption \ref{def_D2}, since it helps eliminate the optimality gap of $d^{1/2}$ appearing in LinTS.

\begin{theorem}[Regret of LinBUCB \textit{with} approximate inference] \label{thm:LinBUCB2}
Suppose that Assumptions \ref{linearassu} and \ref{assu0} hold.

1) Suppose $\Pi_t$ is derived from Assumption \ref{def_D1}. Set $1-\kappa_2\le \gamma <1$ in Algorithm \ref{algo:LinBUCB}. Then with probability $1-\delta$, the regret of LinBUCB with approximate inference is bounded by
$$R(T) \le \beta_T (\delta) \left(\sqrt{c_2d\log \frac{c_2'd}{1-\gamma} } + 1\right) \sqrt{2Td \log (1+\frac{T}{\lambda})}.
$$

2) Suppose $\Pi_t$ is derived from Assumption \ref{def_D2}. Set $1-\kappa_2\le \gamma <1$ in Algorithm \ref{algo:LinBUCB}. Then with probability $1-\delta$, the regret of LinBUCB with approximate inference is bounded by
$$R(T) \le \beta_T (\delta) \left(\hat{c}_2(1-\gamma) + 1\right) \sqrt{2Td \log (1+\frac{T}{\lambda})},
$$
where $\kappa_2=\left(\epsilon\alpha_1 (\alpha_1-1)+1\right)^{\frac{1}{1-\alpha_1}} \kappa_1^{\frac{\alpha_1}{\alpha_1-1}}$, 
$c_2 = c_1 + \frac{\alpha_2-1}{\alpha_2}$, $c_2'=\frac{c_1'}{\left(\epsilon\alpha_2(\alpha_2-1)+1\right)^{\alpha_2}}$, and $\hat{c}_2(\zeta) = \hat{c}_1(\zeta^{\frac{\alpha_2-1}{\alpha_2}} \left(\epsilon\alpha_2(\alpha_2-1)+1\right)^{\alpha_2})$. 
\end{theorem}

Theorem \ref{thm:LinBUCB2} shows that the regret is bounded above by $\tilde{O}(d^{3/2}\sqrt{T})$ with Assumption \ref{def_D1} and $\tilde{O}(d\sqrt{T})$ with Assumption \ref{def_D2}. It indicates that LinBUCB with approximate inference preserves the same rate of the regret upper bound as in LinBUCB without approximate inference (Theorem \ref{thm:LinBUCB1}), with degradation lying in the constant terms: $\frac{1}{\kappa_2}> \frac{1}{\kappa_1}$, $c_2 > c_1$, and $c_2'>c_1'$. These observations are parallel to the ones for LinTS with or without approximate inference (Section \ref{sec:regretLinTS}).

In Theorems \ref{thm:LinBUCB1} and \ref{thm:LinBUCB2}, ideally, we would choose $\gamma = 1- \kappa_1$ or $\gamma = 1- \kappa_2$ in Algorithm \ref{algo:LinBUCB} to obtain a lower regret upper bound. In practice, the exact value of $\kappa_1$ or $\kappa_2$ might not be available, but our theorems hold for any constant quantile larger than $1- \kappa_1$ or $1- \kappa_2$, slightly diminishing the negative effect of unknown $\kappa_1$ or $\kappa_2$.

\textbf{Outline of Proof Techniques.} We present high-level ideas for the technical derivation of our theorems. The detailed proof of all theorems is presented in Appendix \ref{sec:proof}. First, to understand the impact of the approximate inference on the posterior, we develop a sensitivity analysis of the distribution property change with respect to inference error; See Theorems \ref{thm:invariance} - \ref{thm:concentration}. Herein, we show that under Assumption \ref{assu0}, $Q_{t}$ can preserve similar anti-concentration and concentration properties as in $\Pi_{t}$, but with different bounding constants.
A key ingredient is to analyze the shift of the quantile functions of $u^\top \eta$ and $\|\eta\|$ caused by the $\alpha$-divergence error in Assumption \ref{assu0}; See Theorem \ref{thm:concentration}. Based on this sensitivity analysis, the connection between algorithms (LinTS and LinBUCB) with and without approximate inference can be established.
In addition, as a natural generalization of our tools, we can derive upper and lower bounds on the quantiles of $u^\top \eta$ based on the anti-concentration and concentration properties; See Theorem \ref{thm:concentrationquantile}. This result can be used to analyze the gap between the quantiles chosen by LinBUCB and the optimal reward, contributing to its regret analysis.


\subsection{Negative Results} \label{sec:negativeresults}



In this section, we present negative results in a worst-case scenario, demonstrating that linear regret can occur under a single bounded $\alpha$-divergence. Specifically, we show that LinTS and LinBUCB may suffer linear regret in an adversarial setting where their inference errors are constrained using only one $\alpha$-divergence. This observation parallels previous negative results in multi-arm bandits \citep{phan2019thompson,huang2023optimal}. We caution that this finding does not imply that the two bandit algorithms would universally fail across all approximate inference methods. In practice, methods relying on a single bounded $\alpha$-divergence may still yield satisfactory results. Nevertheless, our negative results highlight that employing designs that incorporate two $\alpha$-divergences provides a statistically guaranteed strategy to mitigate worst-case scenarios.

\begin{assumption} \label{assu3}
$Q_{t}$ is a multivariate distribution on $\mathbb{R}^d$ absolutely continuous with respect to the Lebesgue measure, whose support is the same as $\Pi_{t}$'s support. There exists a positive value $\epsilon \in (0,+\infty)$ and a parameter $\alpha$ such that
\begin{equation} \label{equ:error2}
D_{\alpha}(\Pi_{t}, Q_{t})\le \epsilon, \forall t\in [T].  
\end{equation}
\end{assumption}

We establish two negative results with approximate inference under Assumption \ref{assu3}:
\begin{theorem} \label{thm:tsfails}
Consider a stochastic linear bandit problem where the ground-truth parameter $\theta^* = (\mu_1, \mu_2)$ and $\mu_1>\mu_2$.
For any given $\alpha>0$ and any error threshold $\epsilon>0$, there exists a problem instance and a sequence of distributions $Q_{t}$, such that 1) for all $t \ge 1$, $Q_{t}$ satisfies Assumption \ref{assu3}, and 2) LinTS from the approximate distribution $Q_{t}$ has a linear frequentist regret $R(T)=\Omega(T)$.
\end{theorem}

\begin{theorem} \label{thm:bucbfails}
Consider a stochastic linear bandit problem where the ground-truth parameter $\theta^* = (\mu_1, \mu_2)$ and $\mu_1>\mu_2$.
For any given $\alpha>0$ and any error threshold $\epsilon$ satisfying $\epsilon> \frac{1}{\alpha(\alpha-1)}(\gamma^{1-\alpha}-1)$ ($\alpha>0$ and $\alpha\ne 1$) or $\epsilon>-\log(\gamma)$ ($\alpha=1$), there exists a problem instance and a sequence of distributions $Q_{t}$, such that 1) for all $t \ge 1$, $Q_{t}$ satisfies Assumption \ref{assu3}, and 2) LinBUCB from the approximate distribution $Q_{t}$ has a linear frequentist regret $R(T)=\Omega(T)$.
\end{theorem}


The above two theorems demonstrate that bounding a single 
$\alpha$-divergence (Assumption \ref{assu3}) is insufficient to guarantee sublinear regret for LinTS or LinBUCB. This underscores the necessity of Assumption \ref{assu0} that requires two distinct $\alpha$-divergences to be bounded. These results also imply that posterior distributions in Bayesian bandits cannot be arbitrarily chosen. Improperly selected approximate posteriors could lead to the failure of LinTS and LinBUCB. To ensure effective performance, approximate posteriors must strike the right balance between exploration and exploitation.

\section{Experiments} \label{sec:exp}

In this section, we conduct experiments with different problem settings to validate the correctness of our theory. The methods, Linear Thompson Sampling with Approximate Inference and Linear Bayesian Upper Confidence Bound with Approximate Inference, denoted as LinTS$\_$Approximate and LinBUCB$\_$Approximate, are compared with their original versions LinTS and LinBUCB.   

\begin{figure*}[ht]
    \centering
    \includegraphics[width=0.9\textwidth]{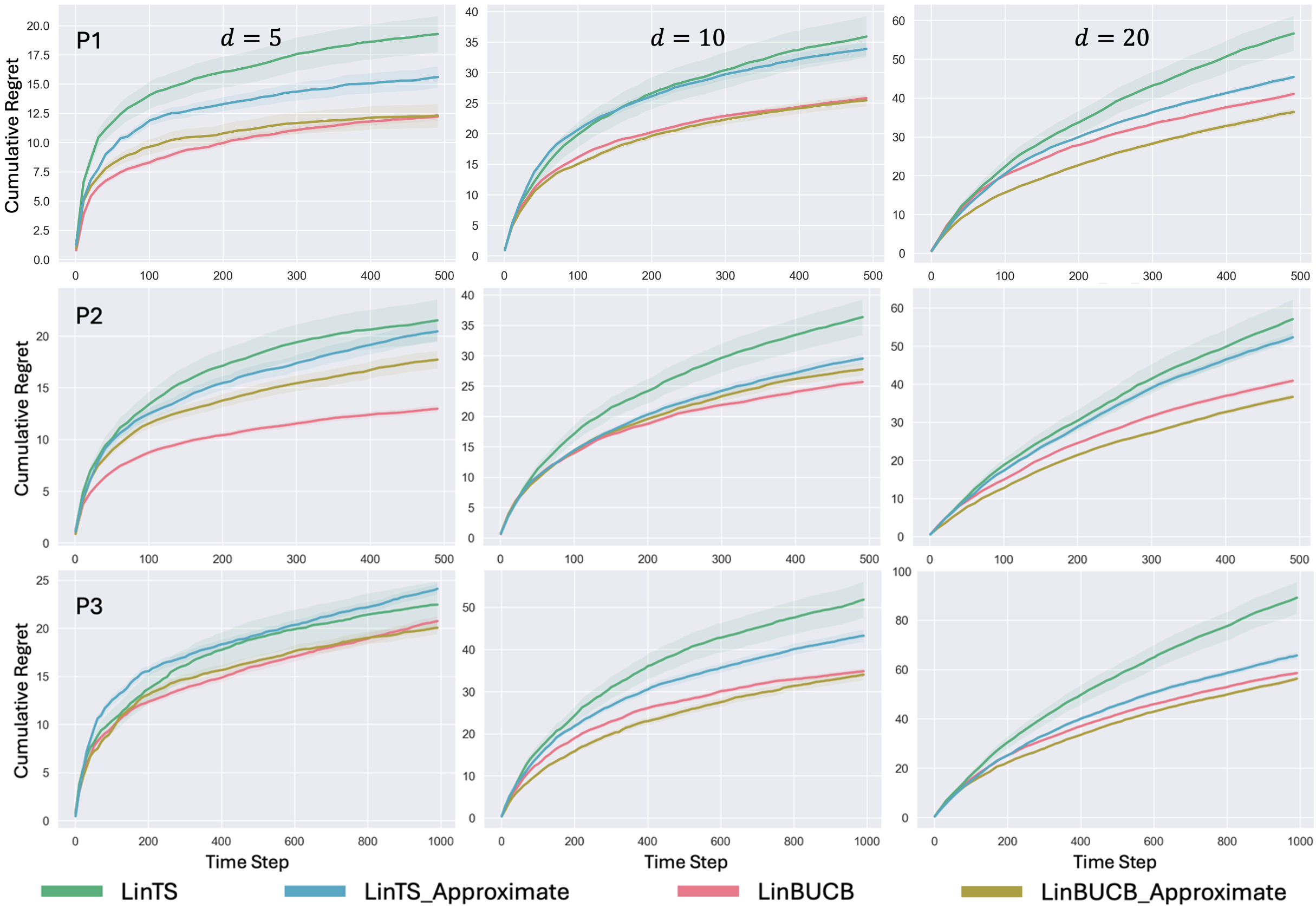}
    \caption{Results of LinBUCB, LinTS, LinBUCB$\_$Approximate, and LinTS$\_$Approximate under different problem settings. Results are averaged over 10 runs with shaded standard errors.} 
    \label{fig:too}
\end{figure*}

We perform experiments on the following three linear contextual bandit problems with various context dimensions: $P1: \theta^*=[1,-1,1,-1,\cdots, (-1)^{d+1}]$, $P2: \theta^*=[\sin(1),\sin(2),\sin(3),\cdots,\sin(d)]$, and $P3: \theta^*$ is a sample generated once from $\text{Uniform}(0, 1)^{\otimes d}$. For all settings, we use a fixed size of the arm set $K = 10$, along with $3$ different context dimensions $d = 5, 10, 20$. 
We set the time horizon as $T = 500, 1000$, which is sufficiently large to ensure the comparisons of all methods. At each time step $t$, $K = 10$ arms are generated where the context vector of each arm $x$ is randomly sampled from a normal distribution $\mathcal{N}(0,I_d)$ and then scaled to be in the unit-norm ball. 
For approximate inference, we implement a diagonal matrix inverse method \citep{zhou2020neural,zhang2021neural} to approximately compute the mean and covariance in the Gaussian distributions. This method largely facilitates the computational time of matrix inverse. More implementation details, including approximate computing approach selection, computation resources, execution time, and hyperparameters are presented in Appendix \ref{sec:implement}.

In Figure \ref{fig:too}, we present the mean and the standard error of the cumulative regret plots, i.e., $R(T)$ against $T$, of different methods over 10 runs. We have the following three observations:


1) In general, all bandit algorithms achieve high performance with a sublinear convergence rate. Particularly, LinBUCB and LinBUCB$\_$Approximate achieve similar regret plots, which validates our theorems that these approaches can preserve their original rate of regret upper bound in the presence of approximate inference. A similar observation also holds for LinTS and LinTS$\_$Approximate. 2) LinBUCB and LinBUCB$\_$Approximate generally outperform LinTS and LinTS$\_$Approximate in most experimental settings, probably due to the conservative over-exploration of Thompson sampling. This also corroborates our theory that LinBUCB has a better rate $\tilde{O}(d\sqrt{T})$ than LinTS under Assumption \ref{def_D2}. 3) LinBUCB$\_$Approximate and LinTS$\_$Approximate are more computationally efficient than LinBUCB and LinTS with reduced time and space complexities. This reveals the advantages of the approximate procedure, as it obviates the need for explicitly computing $V_t^{-1}$ and sampling from complex covariance matrices. 

Finally, more results are presented in Appendix \ref{sec:expmore}. Results on computational efficiency are provided in Appendix~\ref{sec:comput_method} and Appendix~\ref{sec:comput_matrix}, including runtime comparisons across different algorithms and the diagonal matrix inverse versus the full matrix inverse. In Appendix \ref{sec:results}, we present additional experimental results with various configurations of hyper-parameters ($K$, $d$, and $T$) to further demonstrate the scalability of our proposed methodology. Moreover, we conduct the sensitivity analysis to examine the impact of different choices of quantile $\gamma$ in LinBUCB and present the results in Appendix \ref{sec:sensitivity}. Overall, our results corroborate our theoretical findings, offering insights into how approximate inference influences algorithmic performance.

\section{Conclusions}

In this paper, we present a theoretical framework to investigate the impact of approximate inference on Bayesian bandit algorithms in terms of frequentist regret performance. Specifically, we analyze LinTS and LinBUCB, demonstrating that both algorithms can achieve similar regret upper bounds in exact and approximate inference settings, provided their inference errors, measured by two distinct $\alpha$-divergences, are bounded throughout the time horizon. Furthermore, we establish that the regret rate of LinBUCB can be improved to match the minimax optimal rate by introducing a new definition of well-behaved distributions, which provides a refined characterization of the LinBUCB structure. 
Our current study focuses on the linear reward functions. In future work, we aim to extend this framework to a broader class of Bayesian bandit algorithms and investigate its applicability to non-linear reward models.

{
\bibliographystyle{unsrt}
\bibliography{bib}
}

\appendix
\newpage

\section{Proof} \label{sec:proof}

In this section, we prove the results in the main paper.

\subsection{Proofs of Results in Section \ref{sec:LinBUCB}}

\begin{proof}[Proof of Proposition \ref{prop_D2}]
$\mathcal{N}(0, I_{d})$ satisfies Assumption \ref{def_D2} Part 1), or equivalently Assumption \ref{def_D1} Part 1), as shown in \cite{abeille2017linear}. $\mathcal{N}(0, I_{d})$ satisfies Assumption \ref{def_D2} Part 2) because of the following observation. Since $\eta \sim \mathcal{N}(0, I_d)$, then, for any $u \in \mathbb{R}^d $ with $\|u\| = 1$, we have $u^T \eta \sim \mathcal{N}(0, 1)$, the univariate standard Gaussian distribution. Therefore, taking $\hat{c}_1(\delta) = \mathcal{Q}(1-\delta, \mathcal{N}(0, 1))$, which is a constant that only depends on $\delta$ and is independent of $d$, we obtain that $\mathbb{P}_{\eta \sim \mathcal{D}^{2}}(u^\top \eta \leq \hat{c}_1(\delta)) \geq 1-\delta$.
\end{proof}

\subsection{Proofs of Results in Section \ref{sec:regretLinTS}}

In this section, we first establish a few preparatory results that present some of the general properties of $\alpha$-divergence and its connection to the anti-concentration and concentration properties in Assumptions \ref{def_D1} and \ref{def_D2}. Then we prove Theorem \ref{thm:LinTS}.

\begin{theorem} \label{thm:invariance}
Consider any two random vectors $X_1 \in \mathbb{R}^d$ and $X_2 \in \mathbb{R}^d$ with continuous multivariate distributions. \\
1)  Suppose $g: \mathbb{R}^d \to \mathbb{R}$ is an invertible and continuously differentiable map. Then 
$$D_{\alpha}(X_1,X_2)=D_{\alpha}(g(X_1),g(X_2)).$$
In particular, for any vector $a$ and any invertible matrix $B$, we have
$$D_{\alpha}(X_1,X_2)=D_{\alpha}(a+B X_1,a+ B X_2).$$
2) Suppose $g: \mathbb{R}^d \to \mathbb{R}$ is a real-valued continuously differentiable function. Then 
$$D_{\alpha}(g(X_1),g(X_2))\le D_{\alpha}(X_1,X_2).$$
\end{theorem}

\begin{proof}[Proof of Theorem \ref{thm:invariance}]
1) Suppose that the density functions of $X_1$ and $X_2$ are given by $p_1(x)$ and $p_2(x)$, respectively. By definition, the $\alpha$-divergence between $X_1$ and $X_2$ is 
$$D_{\alpha}(X_1,X_2)=\frac{1}{\alpha(\alpha-1)}\left(\int p_1(x)^{\alpha}p_2(x)^{1-\alpha}dx-1\right).$$
As $g: \mathbb{R}^d \to \mathbb{R}$ is an invertible and continuously differentiable map, the density function of $g(X_i)$ is then given by
$$p_i(g^{-1}(y))|\det(J_{g^{-1}}(y))|,$$
where $g^{-1}$ is the inverse function of $g$ and $J_{g^{-1}}(x)$ is the Jacobian matrix of $g$. Therefore, 
\begin{align*}
& D_{\alpha}(g(X_1),g(X_2))\\
= &\frac{1}{\alpha(\alpha-1)}\left(\int \left(p_1(g^{-1}(y))|\det(J_{g^{-1}}(y))|\right)^{\alpha} \left(p_2(g^{-1}(y))|\det(J_{g^{-1}}(y))|\right)^{1-\alpha}dy-1\right) \\   
= &\frac{1}{\alpha(\alpha-1)}\left(\int \left(p_1(g^{-1}(y))\right)^{\alpha} \left(p_2(g^{-1}(y))\right)^{1-\alpha} |\det(J_{g^{-1}}(y))| dy-1\right) \\  
= &\frac{1}{\alpha(\alpha-1)}\left(\int p_1(x)^{\alpha} p_2(x)^{1-\alpha} |\det(J_{g^{-1}}(g(x)))| |\det(J_{g}(x))|  dx-1\right), \\ 
& \quad \text{where we use the change of variables formula: } y = g(x),\\
= &\frac{1}{\alpha(\alpha-1)}\left(\int p_1(x)^{\alpha} p_2(x)^{1-\alpha} dx-1\right), \\ 
& \quad \text{where we use the inverse function theorem: } J_{g^{-1}}(g(x)) = (J_{g}(x))^{-1},\\
= &D_{\alpha}(X_1,X_2).
\end{align*}
In particular, when $g(x)=a+Bx$, where $a$ is a vector and $B$ is an invertible matrix, $g$ is a continuously differentiable map and its inverse
$g^{-1}(y) = B^{-1} y - B^{-1} a$ is also a continuously differentiable map. The density function of $g(X_i)$ is then given by
$$p_i(B^{-1} y - B^{-1} a)|\det(B^{-1})|.$$
The above argument holds for $g(x)=a+Bx$. Therefore, we have
$$D_{\alpha}(X_1,X_2)=D_{\alpha}(a+B X_1,a+ B X_2).$$

2) As $g: \mathbb{R}^d \to \mathbb{R}$ is a real-valued continuously differentiable function, the density function of $g(X_i)$ is then given by
$$\tilde{p}_i(y)=\int p_i(x) \delta(y-g(x))dx,$$
where $\delta$ is the Dirac delta function. 

Note that the function
$$h(z_1, z_2) = \frac{1}{\alpha(\alpha-1)} z_1 ^\alpha z_2^{1-\alpha}$$
is a convex function on $[0,\infty)^2$, as its Hessian given by
$$H_{h} = \begin{bmatrix}
z_1 ^{\alpha-2} z_2^{1-\alpha} & -z_1 ^{\alpha-1} z_2^{-\alpha} \\
-z_1 ^{\alpha-1} z_2^{-\alpha} & z_1 ^{\alpha} z_2^{-\alpha-1}
\end{bmatrix}$$
is a positive semi-definite matrix. Applying the multivariate Jensen’s inequality to $h(z_1, z_2)$, we obtain
\begin{align*}
&\frac{1}{\alpha(\alpha-1)} \left( \frac{1}{|\{x:g(x)=y\}|}\int p_1(x) \delta(y-g(x))dx\right)^{\alpha} \left(\frac{1}{|\{x:g(x)=y\}|} \int p_2(x) \delta(y-g(x))dx\right)^{1-\alpha}\\
\le & \frac{1}{\alpha(\alpha-1)}  \frac{1}{|\{x:g(x)=y\}|} \int \left(p_1(x)\delta(y-g(x))\right)^{\alpha} \left(p_2(x)\delta(y-g(x))\right)^{1-\alpha} dx,
\end{align*}
which leads to
\begin{align*}
&\frac{1}{\alpha(\alpha-1)} \tilde{p}_1(y)^{\alpha} \tilde{p}_2(y)^{1-\alpha}\\
\le & \frac{1}{\alpha(\alpha-1)}  \int p_1(x)^{\alpha} p_2(x)^{1-\alpha} \delta(y-g(x)) dx.
\end{align*}
It then follows that
\begin{align*}
& D_{\alpha}(g(X_1),g(X_2))\\
= &\frac{1}{\alpha(\alpha-1)}\left(\int \tilde{p}_1(y)^{\alpha} \tilde{p}_2(y)^{1-\alpha}dy-1\right) \\   
\le & \frac{1}{\alpha(\alpha-1)}\left( \int \left(\int p_1(x)^{\alpha} p_2(x)^{1-\alpha} \delta(y-g(x)) dx \right) dy -1\right)\\
= & \frac{1}{\alpha(\alpha-1)}\left( \int \int p_1(x)^{\alpha} p_2(x)^{1-\alpha} \delta(y-g(x)) dydx -1\right),\\
& \quad \text{where we use the Fubini's theorem,}\\
= & \frac{1}{\alpha(\alpha-1)}\left( \int p_1(x)^{\alpha} p_2(x)^{1-\alpha} dx -1\right),\\
& \quad \text{where the inner integration is over $\delta(y-g(x))$ of $y$,}\\
= &D_{\alpha}(X_1,X_2).
\end{align*}
\end{proof}

Before moving to the next result, we recall the following theorem that is adopted from \cite{huang2023optimal}.
\begin{theorem} \label{thm:quantileshift} [Theorem 3.4 in \cite{huang2023optimal}] 
Consider any two continuous univariate distributions $P_1$ and $P_2$. 
Let $0<\gamma<1$. Let $\delta_{\gamma,\epsilon}$ satisfy that $\mathcal{Q}(\gamma,P_1)=\mathcal{Q}(\gamma+\delta_{\gamma,\epsilon},P_2)$ where $-\gamma\le \delta_{\gamma,\epsilon}\le 1-\gamma$.

1) If $D_{\alpha}(P_1,P_2)\le \epsilon$ where $\alpha>1$, then
$$\delta_{\gamma,\epsilon}\le 1-\gamma-\left(\epsilon\alpha(\alpha-1)+1\right)^{\frac{1}{1-\alpha}}\left(1-\gamma\right)^{\frac{\alpha}{\alpha-1}}.$$
Note that when $\alpha>1$, $\left(\epsilon\alpha(\alpha-1)+1\right)^{\frac{1}{1-\alpha}}<1$ and $\left(1-\gamma\right)^{\frac{\alpha}{\alpha-1}}<1-\gamma$.

2) If $D_{\alpha}(P_1,P_2)\le \epsilon$ where $\alpha<0$, then
$$\delta_{\gamma,\epsilon}\ge 1-\gamma-\left(\epsilon\alpha(\alpha-1)+1\right)^{\frac{1}{1-\alpha}}\left(1-\gamma\right)^{\frac{\alpha}{\alpha-1}}.$$
Note that when $\alpha<0$, $\left(\epsilon\alpha(\alpha-1)+1\right)^{\frac{1}{1-\alpha}}>1$ and $\left(1-\gamma\right)^{\frac{\alpha}{\alpha-1}}>1-\gamma$.

3) Suppose that $\alpha\in (0,1)$ and $\epsilon\ge \frac{-1}{\alpha(\alpha-1)}$. Then for any $\delta_{\gamma,\epsilon}\in [-\gamma,1-\gamma]$, there exist two distributions $P_1$ and $P_2$ such that
$D_{\alpha}(P_1,P_2)\le \epsilon.$
This implies that the condition $D_{\alpha}(P_1,P_2)\le \epsilon$ cannot control the quantile shift between $P_1$ and $P_2$ in general when $\alpha\in (0,1)$.
\end{theorem}

Theorem \ref{thm:quantileshift} is distribution-free: the bound of $\delta_{\gamma,\epsilon}$ does not depend on any specific distributions (noting that distribution changes as $t$ evolves in bandit problems). In particular, as the quantile function $\mathcal{Q}(\cdot,P_2)$ is an increasing function, 
Theorem \ref{thm:quantileshift} Part 1) implies that
$$\mathcal{Q}(\gamma,P_1)\le \mathcal{Q}(1-\left(\epsilon\alpha(\alpha-1)+1\right)^{\frac{1}{1-\alpha}}\left(1-\gamma\right)^{\frac{\alpha}{\alpha-1}},P_2),$$ when $\alpha>1$ and Theorem \ref{thm:quantileshift} Part 2) implies that
$$\mathcal{Q}(\gamma,P_1)\ge \mathcal{Q}(1-\left(\epsilon\alpha(\alpha-1)+1\right)^{\frac{1}{1-\alpha}}\left(1-\gamma\right)^{\frac{\alpha}{\alpha-1}},P_2),$$ when $\alpha<0$.

\begin{theorem} \label{thm:concentration}
Consider any two random vectors $X_1 \in \mathbb{R}^d$, and $X_2 \in \mathbb{R}^d$ with continuous multivariate distributions. Suppose we are given any vector $a$ and any invertible matrix $B$.
\\
1) (Anti-concentration preservation with $\alpha$-divergence) Suppose that there exists a strictly positive $\kappa_1>0$ such that for any $u \in \mathbb{R}^d $ with $\|u\| = 1$, 
$$\mathbb{P}(u^\top (B^{-1}X_1- B^{-1}a) \geq 1) \geq \kappa_1.$$ 
If $D_{\alpha}(X_1,X_2)\le \epsilon$ where $\alpha>1$, then we have
$$\mathbb{P}(u^\top (B^{-1}X_2- B^{-1}a) \geq 1) \geq \kappa_2=\left(\epsilon\alpha(\alpha-1)+1\right)^{\frac{1}{1-\alpha}} \kappa_1^{\frac{\alpha}{\alpha-1}},$$
for any $u \in \mathbb{R}^d $ with $\|u\| = 1$.

2) (Concentration Type-I preservation with $\alpha$-divergence) Suppose that there exist two positive constants $c_1>0$ and $c_1'>0$, such that for any
$\delta \in (0, 1)$ 
$$
\mathbb{P}\left(\|B^{-1}X_1- B^{-1}a\| \leq \sqrt{c_1d\log \frac{c_1'd}{\delta} }\right)\geq 1 -\delta,
$$
If $D_{\alpha}(X_1,X_2)\le \epsilon$ where $\alpha<0$, then we have, for any $\delta>0$,
$$\mathbb{P}\left(\|B^{-1}X_2- B^{-1}a\| \leq \sqrt{c_2d\log \frac{c_2'd}{\delta} }\right)\geq 1 -\delta,$$
where $c_2 = c_1 + \frac{\alpha-1}{\alpha}$ and $c_2'=\frac{c_1'}{\left(\epsilon\alpha(\alpha-1)+1\right)^{\alpha}}$.

3) (Concentration Type-II preservation with $\alpha$-divergence) Suppose that for any $\delta>0$, there exists a constant $\hat{c}_1(\delta)>0$ (which is free of $d$ and only depending on $\delta$) such that for any $u \in \mathbb{R}^d $ with $\|u\| = 1$,
$$
\mathbb{P}(u^\top (B^{-1}X_1- B^{-1}a) \leq \hat{c}_1(\delta)) \geq 1-\delta.
$$
If $D_{\alpha}(X_1,X_2)\le \epsilon$ where $\alpha<0$, then we have, for any $\delta>0$,
$$\mathbb{P}(u^\top (B^{-1}X_2- B^{-1}a) \leq \hat{c}_2 (\delta)) \geq 1-\delta,$$
for any $u \in \mathbb{R}^d $ with $\|u\| = 1$ 
where $\hat{c}_2(\delta) = \hat{c}_1(\delta^{\frac{\alpha-1}{\alpha}} \left(\epsilon\alpha(\alpha-1)+1\right)^{\alpha})$.

\end{theorem}

\begin{proof}[Proof of Theorem \ref{thm:concentration}]
1) By Theorem \ref{thm:invariance} Part 1), we have
$$D_{\alpha}(X_1,X_2)=D_{\alpha}(B^{-1}X_1- B^{-1}a,B^{-1}X_2- B^{-1}a).$$ 
By Theorem \ref{thm:invariance} Part 2), we have, for any $u \in \mathbb{R}^d $ with $\|u\| = 1$,
$$D_{\alpha}(u^\top (B^{-1}X_1- B^{-1}a), u^\top (B^{-1}X_2- B^{-1}a))\le D_{\alpha}(B^{-1}X_1- B^{-1}a,B^{-1}X_2- B^{-1}a).$$ 
Let $Y_1 = u^\top (B^{-1}X_1- B^{-1}a)$ and $Y_2 = u^\top (B^{-1}X_2- B^{-1}a)$. The above argument indicates that
$$D_{\alpha}(Y_1,Y_2)\le D_{\alpha}(X_1, X_2)\le \epsilon.$$
Note that 
$$\mathbb{P}(Y_1 \geq 1)=\mathbb{P}(u^\top (B^{-1}X_1- B^{-1}a) \geq 1) \geq \kappa_1$$
is equivalent to 
$$\mathbb{P}(Y_1 \leq 1)=\mathbb{P}(u^\top (B^{-1}X_1- B^{-1}a) \leq 1) \leq 1-\kappa_1.$$
This implies that the $(1-\kappa_1)$-quantile of $Y_1$ satisfies
$$\mathcal{Q}(1-\kappa_1,Y_1)\ge 1.$$
By Theorem \ref{thm:quantileshift} Part 1), we have that
\begin{align*}
1\le & \ \mathcal{Q}(1-\kappa_1,Y_1)\\
\le & \ \mathcal{Q}(1-\left(\epsilon\alpha(\alpha-1)+1\right)^{\frac{1}{1-\alpha}} \kappa_1^{\frac{\alpha}{\alpha-1}},Y_2),
\end{align*}
since $\mathcal{Q}(\cdot, Y_2)$ is an increasing function. This implies that
$$\mathbb{P}(Y_2 \leq 1) \leq 1-\left(\epsilon\alpha(\alpha-1)+1\right)^{\frac{1}{1-\alpha}} \kappa_1^{\frac{\alpha}{\alpha-1}},$$
which is equivalent to 
$$\mathbb{P}(Y_2 \geq 1) \geq \left(\epsilon\alpha(\alpha-1)+1\right)^{\frac{1}{1-\alpha}} \kappa_1^{\frac{\alpha}{\alpha-1}}.$$
Hence,
$$\mathbb{P}(u^\top (B^{-1}X_2- B^{-1}a) \geq 1) \geq \kappa_2=\left(\epsilon\alpha(\alpha-1)+1\right)^{\frac{1}{1-\alpha}} \kappa_1^{\frac{\alpha}{\alpha-1}},$$
for any $u \in \mathbb{R}^d $ with $\|u\| = 1$.

2) By Theorem \ref{thm:invariance} Part 1), we have
$$D_{\alpha}(X_1,X_2)=D_{\alpha}(B^{-1}X_1- B^{-1}a,B^{-1}X_2- B^{-1}a).$$ 
By Theorem \ref{thm:invariance} Part 2), we have
$$D_{\alpha}(\|B^{-1}X_1- B^{-1}a\|, \|B^{-1}X_2- B^{-1}a\|)\le D_{\alpha}(B^{-1}X_1- B^{-1}a,B^{-1}X_2- B^{-1}a).$$ 
Let $Y_1 = \|B^{-1}X_1- B^{-1}a\|$ and $Y_2 = \|B^{-1}X_2- B^{-1}a\|$. The above argument indicates that
$$D_{\alpha}(Y_1,Y_2)\le D_{\alpha}(X_1, X_2)\le \epsilon.$$
Note that 
$$\mathbb{P}\left(Y_1 \leq \sqrt{c_1d\log \frac{c_1'd}{\delta} }\right)=\mathbb{P}\left(\|B^{-1}X_1- B^{-1}a\| \leq \sqrt{c_1d\log \frac{c_1'd}{\delta} }\right) \geq 1-\delta.$$
This implies that the $(1-\delta)$-quantile of $Y_1$ satisfies
$$\mathcal{Q}(1-\delta,Y_1)\le \sqrt{c_1d\log \frac{c_1'd}{\delta} }. $$
By Theorem \ref{thm:quantileshift} Part 2), we have that
\begin{align*}
\sqrt{c_1d\log \frac{c_1'd}{\delta} }\ge & \ \mathcal{Q}(1-\delta,Y_1)\\
\ge & \ \mathcal{Q}(1-\left(\epsilon\alpha(\alpha-1)+1\right)^{\frac{1}{1-\alpha}} \delta^{\frac{\alpha}{\alpha-1}},Y_2),
\end{align*}
since $\mathcal{Q}(\cdot, Y_2)$ is an increasing function. 
Let $\tilde{\delta}= \left(\epsilon\alpha(\alpha-1)+1\right)^{\frac{1}{1-\alpha}} \delta^{\frac{\alpha}{\alpha-1}}$, or equivalently
$$\delta=(\tilde{\delta})^{\frac{\alpha-1}{\alpha}} \left(\epsilon\alpha(\alpha-1)+1\right)^{\alpha}.$$
Note that for any $\tilde{\delta} \in (0,1)$, we can always find such $\delta \in (0,1)$ satisfying the above equality. Hence,
\begin{align*}
\mathcal{Q}(1-\tilde{\delta},Y_2) &\le \sqrt{c_1d\log \frac{c_1'd}{(\tilde{\delta})^{\frac{\alpha-1}{\alpha}} \left(\epsilon\alpha(\alpha-1)+1\right)^{\alpha}} }\\
&= \sqrt{c_1d\log \frac{c_1'd}{\left(\epsilon\alpha(\alpha-1)+1\right)^{\alpha}} + {\frac{\alpha-1}{\alpha} \log \frac{1}{\tilde{\delta}} }}\\
&\le \sqrt{c_1d\log \frac{c_1'd}{\left(\epsilon\alpha(\alpha-1)+1\right)^{\alpha}} + d {\frac{\alpha-1}{\alpha} \log \frac{1}{\tilde{\delta}} }}\\
&= \sqrt{c_2d\log \frac{c_2'd}{\tilde{\delta}} },
\end{align*}
where $c_2 = c_1 + \frac{\alpha-1}{\alpha}$ and $c_2'=\frac{c_1'}{\left(\epsilon\alpha(\alpha-1)+1\right)^{\alpha}}$.

This implies that for any $\delta\in(0,1)$,
$$\mathbb{P}\left(Y_2 \leq \sqrt{c_2d\log \frac{c_2'd}{\delta} }\right)\ge 1-\delta.$$

3) By Theorem \ref{thm:invariance} Part 1), we have
$$D_{\alpha}(X_1,X_2)=D_{\alpha}(B^{-1}X_1- B^{-1}a,B^{-1}X_2- B^{-1}a).$$ 
By Theorem \ref{thm:invariance} Part 2), we have, for any $u \in \mathbb{R}^d $ with $\|u\| = 1$,
$$D_{\alpha}(u^\top (B^{-1}X_1- B^{-1}a), u^\top (B^{-1}X_2- B^{-1}a))\le D_{\alpha}(B^{-1}X_1- B^{-1}a,B^{-1}X_2- B^{-1}a).$$ 
Let $Y_1 = u^\top (B^{-1}X_1- B^{-1}a)$ and $Y_2 = u^\top (B^{-1}X_2- B^{-1}a)$. The above argument indicates that
$$D_{\alpha}(Y_1,Y_2)\le D_{\alpha}(X_1, X_2)\le \epsilon.$$
Note that 
$$\mathbb{P}(Y_1 \leq \hat{c}_1(\delta))=\mathbb{P}(u^\top (B^{-1}X_1- B^{-1}a) \leq \hat{c}_1(\delta)) \ge 1-\delta,$$
which implies that the $(1-\delta)$-quantile of $Y_1$ satisfies
$$\mathcal{Q}(1-\delta,Y_1)\le \hat{c}_1(\delta).$$
By Theorem \ref{thm:quantileshift} Part 2), we have that
\begin{align*}
\hat{c}_1(\delta) \ge & \ \mathcal{Q}(1-\delta,Y_1)\\
\ge & \ \mathcal{Q}(1-\left(\epsilon\alpha(\alpha-1)+1\right)^{\frac{1}{1-\alpha}} \delta^{\frac{\alpha}{\alpha-1}},Y_2),
\end{align*}
since $\mathcal{Q}(\cdot, Y_2)$ is an increasing function. Let $\tilde{\delta}= \left(\epsilon\alpha(\alpha-1)+1\right)^{\frac{1}{1-\alpha}} \delta^{\frac{\alpha}{\alpha-1}}$, or equivalently
$$\delta=(\tilde{\delta})^{\frac{\alpha-1}{\alpha}} \left(\epsilon\alpha(\alpha-1)+1\right)^{\alpha}.$$
Note that for any $\tilde{\delta} \in (0,1)$, we can always find such $\delta \in (0,1)$ satisfying the above equality. Hence,
$$\mathbb{P}(Y_2 \leq \hat{c}_1((\tilde{\delta})^{\frac{\alpha-1}{\alpha}} \left(\epsilon\alpha(\alpha-1)+1\right)^{\alpha})) \geq 1-\tilde{\delta}.$$
This implies that for any $\delta\in(0,1)$,
$$\mathbb{P}(u^\top (B^{-1}X_2- B^{-1}a) \leq \hat{c}_2(\delta)) \geq 1-\delta,$$
for any $u \in \mathbb{R}^d $ with $\|u\| = 1$ where $\hat{c}_2(\delta) = \hat{c}_1(\delta^{\frac{\alpha-1}{\alpha}} \left(\epsilon\alpha(\alpha-1)+1\right)^{\alpha})$.

\end{proof}

Now we are ready to prove Theorem \ref{thm:LinTS}.

\begin{proof}[Proof of Theorem \ref{thm:LinTS}]
Note that $\tilde{\theta}_t \sim Q_t$. We write another random variable $\tilde{\theta}'_t \sim \Pi_t$ that follows the distribution $\Pi_t$. By Assumption \ref{assu0}, we have
$D_{\alpha_1}(\tilde{\theta}'_t,\tilde{\theta}_t)\le \epsilon$ where $\alpha_1>1$. By the definition of $\Pi_t$ (Assumption \ref{def_D1}), we know that $\beta_t(\delta')^{-1} V_t^{1/2}(\tilde{\theta}'_t-\hat{\theta}_t)$ satisfies that 
for any $u \in \mathbb{R}^d $ with $\|u\| = 1$, 
$$\mathbb{P}(u^\top \beta_t(\delta')^{-1} V_t^{1/2}(\tilde{\theta}'_t-\hat{\theta}_t) \geq 1) \geq \kappa_1.$$ 
By Theorem \ref{thm:concentration} Part 1), we have a similar anti-concentration property for $\beta_t(\delta')^{-1} V_t^{1/2}(\tilde{\theta}_t-\hat{\theta}_t)$, with a different constant term: For any $u \in \mathbb{R}^d $ with $\|u\| = 1$, 
$$\mathbb{P}(u^\top \beta_t(\delta')^{-1} V_t^{1/2}(\tilde{\theta}_t-\hat{\theta}_t) \geq 1) \geq \kappa_2,$$
where $\kappa_2=\left(\epsilon\alpha_1(\alpha_1-1)+1\right)^{\frac{1}{1-\alpha_1}} \kappa_1^{\frac{\alpha_1}{\alpha_1-1}}$.

On the other side, by Assumption \ref{assu0}, we have
$D_{\alpha_2}(\tilde{\theta}'_t,\tilde{\theta}_t)\le \epsilon$ where $\alpha_2<0$. By the definition on $\Pi_t$ (Assumption \ref{def_D1}), we know that $\beta_t(\delta')^{-1} V_t^{1/2}(\tilde{\theta}'_t-\hat{\theta}_t)$ satisfies that 
for any $\zeta>0$, 
$$\mathbb{P}\left(\|\beta_t(\delta')^{-1} V_t^{1/2}(\tilde{\theta}'_t-\hat{\theta}_t)\| \leq \sqrt{c_1d\log \frac{c_1'd}{\zeta} }\right) \geq 1-\zeta.$$ 
By Theorem \ref{thm:concentration} Part 2), we have a similar concentration property for $\beta_t(\delta')^{-1} V_t^{1/2}(\tilde{\theta}_t-\hat{\theta}_t)$, with a different constant term: For any $\zeta>0$, 
$$\mathbb{P}\left(\|\beta_t(\delta')^{-1} V_t^{1/2}(\tilde{\theta}_t-\hat{\theta}_t)\| \leq \sqrt{c_2d\log \frac{c_2'd}{\zeta} }\right) \geq 1-\zeta,$$
where $c_2 = c_1 + \frac{\alpha_2-1}{\alpha_2}$ and $c_2'=\frac{c_1'}{\left(\epsilon\alpha_2(\alpha_2-1)+1\right)^{\alpha_2}}$.

Next, we can apply Theorem 1 in \cite{abeille2017linear} on the approximate sample $\tilde{\theta}_t \sim Q_t$, since the above argument indicates that $\beta_t(\delta')^{-1} V_t^{1/2}(\tilde{\theta}_t-\hat{\theta}_t)$ satisfies similar anti-concentration and concentration properties in Assumption \ref{def_D1}, while the only difference is the constant terms. This results in the conclusion:
$$R(T) \le \left( \beta_T (\delta') + \hat{\gamma}_T (\delta')(1 + \frac{4}{\kappa_2})\right)\sqrt{2Td \log (1+\frac{T}{\lambda})} + \frac{4\hat{\gamma}_T (\delta')}{\kappa_2}
\sqrt{\frac{8T}{\lambda}\log \frac{4}{\delta}},
$$
with probability $1-\delta$.

\end{proof}

\subsection{Proofs of Results in Section \ref{sec:regretLinBUCB}}

We first establish a preparatory result that connects quantile functions to the anti-concentration and concentration properties in Assumptions \ref{def_D1} and \ref{def_D2}. Then we prove Theorems \ref{thm:LinBUCB1} and \ref{thm:LinBUCB2}.

\begin{theorem}\label{thm:concentrationquantile}
Consider any random vector $X_1 \in \mathbb{R}^d$ with continuous multivariate distributions. Suppose we are given any vector $a$ and any invertible matrix $B$.
\\
1) (Anti-concentration linked to quantile.) Suppose that there exists a strictly positive $\kappa_1>0$ such that for any $u \in \mathbb{R}^d $ with $\|u\| = 1$, 
$$\mathbb{P}(u^\top (B^{-1}X_1- B^{-1}a) \geq 1) \geq \kappa_1.$$ 
Then for any $u \in \mathbb{R}^d $ with $u \ne 0$, the $(1-\kappa_1)$-quantile of $u^\top X_1$ satisfies
$$\mathcal{Q}(1-\kappa_1, u^\top X_1)\ge \|u^\top B\|+ u^\top a.$$

2) (Concentration Type-I linked to quantile) Suppose that there exist two positive constants $c_1>0$ and $c_1'>0$ such that for any $\delta \in (0, 1)$ 
$$
\mathbb{P}\left(\|B^{-1}X_1- B^{-1}a\| \leq \sqrt{c_1d\log \frac{c_1'd}{\delta} }\right)\geq 1 -\delta.
$$
Then for any $u \in \mathbb{R}^d $ with $u \ne 0$, the $(1-\delta)$-quantile of $u^\top X_1$ satisfies
$$\mathcal{Q}(1-\delta, u^\top X_1)\le \|u^\top B\|\sqrt{c_1d\log \frac{c_1'd}{\delta} } + u^\top a.$$

3) (Concentration Type-II linked to quantile) Suppose that for any $\delta>0$, there exists a constant $\hat{c}_1>0$ (which is free of $d$ and only depends on $\delta$) such that for any $u \in \mathbb{R}^d $ with $\|u\| = 1$,
$$
\mathbb{P}(u^\top (B^{-1}X_1- B^{-1}a) \leq \hat{c}_1(\delta)) \geq 1-\delta.
$$
Then for any $u \in \mathbb{R}^d $ with $u \ne 0$, the $(1-\delta)$-quantile of $u^\top X_1$ satisfies
$$\mathcal{Q}(1-\delta, u^\top X_1)\le \|u^\top B\| \hat{c}_1(\delta) + u^\top a.$$
\end{theorem}

\begin{proof}[Proof of Theorem \ref{thm:concentrationquantile}]
1) First we note that for any continuous random variable $X_1$,
$$\mathcal{Q}(\cdot, s_1+ s_2 X_1)= s_2 \mathcal{Q}(\cdot, X_1) + s_1,$$
for any $s_2>0$ and $s_1\in \mathbb{R}$.
Moreover, $$\mathbb{P}(u^\top (B^{-1}X_1- B^{-1}a) \geq 1) \geq \kappa_1$$ is equivalent to 
$$\mathcal{Q}(1-\kappa_1, u^\top (B^{-1}X_1- B^{-1}a)) \geq 1.$$
Therefore, we have
\begin{align*}
\mathcal{Q}(1-\kappa_1, u^\top X_1)=& \mathcal{Q}\left(1-\kappa_1, \|u^\top B\| \frac{u^\top B}{\|u^\top B\|} (B^{-1}X_1- B^{-1}a)+ u^\top a\right)\\
=& \|u^\top B\|\mathcal{Q}\left(1-\kappa_1, \frac{u^\top B}{\|u^\top B\|} (B^{-1}X_1- B^{-1}a)\right)+ u^\top a\\
\ge & \|u^\top B\|+ u^\top a.
\end{align*}

2) The argument is similar to Part 1). We first note that 
for any $u \in \mathbb{R}^d $ with $u \ne 0$,
$$\|u^\top B (B^{-1}X_1- B^{-1}a)\| \le \|u^\top B\| \|B^{-1}X_1- B^{-1}a\|,$$
and thus, by the monotonic property of the quantile function,
$$\mathcal{Q}(\cdot, u^\top B (B^{-1}X_1- B^{-1}a))\le \mathcal{Q}(\cdot, \|u^\top B\| \|B^{-1}X_1- B^{-1}a\|).$$
Therefore, we have
\begin{align*}
\mathcal{Q}(1-\delta, u^\top X_1)= & \mathcal{Q}(1-\delta, u^\top B (B^{-1}X_1- B^{-1}a)+ u^\top a)\\
=& \mathcal{Q}(1-\delta, u^\top B (B^{-1}X_1- B^{-1}a)) + u^\top a\\
\le & \mathcal{Q}(1-\delta, \|u^\top B\| \|B^{-1}X_1- B^{-1}a\|) + u^\top a\\
=& \|u^\top B\| \mathcal{Q}(1-\delta, \|B^{-1}X_1- B^{-1}a\|) + u^\top a\\
\le & \|u^\top B\|\sqrt{c_1d\log \frac{c_1'd}{\delta} } + u^\top a.
\end{align*}

3) The argument is similar to Parts 1) and 2). We have
\begin{align*}
\mathcal{Q}(1-\kappa_1, u^\top X_1)= & \mathcal{Q}(1-\kappa_1, u^\top B (B^{-1}X_1- B^{-1}a)+ u^\top a)\\
= & \|u^\top B\|\mathcal{Q}\left(1-\kappa_1, \frac{u^\top B}{\|u^\top B\|} (B^{-1}X_1- B^{-1}a)\right)+ u^\top a\\
\le & \|u^\top B\| \hat{c}_1(\delta) + u^\top a.
\end{align*}
\end{proof}

Now we are ready to prove Theorems \ref{thm:LinBUCB1} and \ref{thm:LinBUCB2}.

\begin{proof}[Proof of Theorem \ref{thm:LinBUCB1}]
1) According to Assumption \ref{def_D1}, apply Theorem \ref{thm:concentrationquantile} Part 1) to $\beta_t(\delta)^{-1} V_t^{1/2}(\tilde{\theta}_t-\hat{\theta}_t)$ with $X_1=\tilde{\theta}_t$, $B = \beta_t(\delta) V_t^{-1/2}$, and $a=\hat{\theta}_t$. We obtain that for any $x\in \mathcal{X}_t$,
\begin{align}
\mathcal{Q}(\gamma, x^\top \tilde{\theta}_t) \ge& \mathcal{Q}(1-\kappa_1, x^\top \tilde{\theta}_t)
\\
\ge &\|x^\top (\beta_t(\delta) V_t^{-1/2})\|+ x^\top \hat{\theta}_t \nonumber\\
=& \beta_t(\delta) \|x\|_{V_t^{-1}} + x^\top \hat{\theta}_t \nonumber\\
\ge& x^\top \theta^*, \label{quantilelower}
\end{align}
where the last inequality follows from Equation \eqref{abba_eq}.
Similarly, applying Theorem \ref{thm:concentrationquantile} Part 2) and letting $\delta= 1-\gamma$, we have
\begin{align}
\mathcal{Q}(\gamma, x^\top \tilde{\theta}_t) \le& \|x^\top (\beta_t(\delta) V_t^{-1/2})\| \sqrt{c_1d\log \frac{c_1'd}{1-\gamma} } + x^\top \hat{\theta}_t\nonumber\\
=& \beta_t(\delta) \|x\|_{V_t^{-1}} \sqrt{c_1d\log \frac{c_1'd}{1-\gamma} } + x^\top \hat{\theta}_t. \label{quantileupper}
\end{align}

At each time step $t$, the regret can be bounded above by
\begin{align*}
J_t(\theta^*)-x_t^\top \theta^* & =  (x^*_t)^\top \theta^* -x_t^\top \theta^*   \\
& \le  \mathcal{Q}(\gamma, (x^*_t)^\top \tilde{\theta}_t) -x_t^\top \theta^*,\\
& \quad \quad \text{where we use Equation } \eqref{quantilelower},\\
& \le  \mathcal{Q}(\gamma, x_t^\top \tilde{\theta}_t) -x_t^\top \theta^*, \\
& \quad \quad \text{where we use the optimality of $x_t$ chosen by LinBUCB: } x_t = \argmax{x\in \mathcal{X}_t} \mathcal{Q}(\gamma, x^\top\tilde{\theta}_t),\\
& \le  \beta_t(\delta) \|x_t\|_{V_t^{-1}} \sqrt{c_1d\log \frac{c_1'd}{1-\gamma} } + x_t^\top \hat{\theta}_t -x_t^\top \theta^*, \\
& \quad \quad \text{where we use Equation } \eqref{quantileupper},\\
& \le  \beta_t(\delta) \|x_t\|_{V_t^{-1}} \sqrt{c_1d\log \frac{c_1'd}{1-\gamma} } + \|x_t\|_{V_t^{-1}} \beta_t(\delta), \\
& \quad \quad \text{where we use Equation } \eqref{abba_eq},\\
& =  \beta_t(\delta) \|x_t\|_{V_t^{-1}} \left(\sqrt{c_1d\log \frac{c_1'd}{1-\gamma} } + 1\right). 
\end{align*}
Therefore, the cumulative regret can be bounded as follows:
\begin{align*}
R(T) &= \sum_{t=1}^{T} (J_t(\theta^*)-x_t^\top \theta^*)  \\
&\le \sum_{t=1}^{T}  \beta_t(\delta) \|x_t\|_{V_t^{-1}} \left(\sqrt{c_1d\log \frac{c_1'd}{1-\gamma} } + 1\right)\\
& \le \beta_T (\delta) \left(\sqrt{c_1d\log \frac{c_1'd}{1-\gamma} } + 1\right) \sqrt{2Td \log (1+\frac{T}{\lambda})},
\end{align*}
where the last inequality follows from Equation \eqref{abba_eq2} and Cauchy–Schwarz inequality. Hence, we obtain the results.

2) The proof is similar to Part 1). The difference is that under Assumption \ref{def_D2}, we now apply Theorem \ref{thm:concentrationquantile} Part 3), instead of Part 2), for the upper bound of the quantile. Specifically, we have
\begin{align}
\mathcal{Q}(\gamma, x^\top \tilde{\theta}_t) \le& \|x^\top (\beta_t(\delta) V_t^{-1/2})\| \hat{c}_1(1-\gamma) + x^\top \hat{\theta}_t\nonumber\\
=& \beta_t(\delta) \|x\|_{V_t^{-1}} \hat{c}_1(1-\gamma) + x^\top \hat{\theta}_t. \label{quantileD2}
\end{align}
Note that this upper bound removes the extra $d^{1/2}$ factor compared with Part 1). Using a similar derivation in Part 1), Equation \eqref{quantileD2} gives 
$$J_t(\theta^*)-x_t^\top \theta^*\le \beta_t(\delta) \|x_t\|_{V_t^{-1}} \left(\hat{c}_1(1-\gamma) + 1\right).$$
Hence we obtain the result as follows:
\begin{align*}
R(T) &= \sum_{t=1}^{T} (J_t(\theta^*)-x_t^\top \theta^*)  \\
&\le \sum_{t=1}^{T}  \beta_t(\delta) \|x_t\|_{V_t^{-1}} \left(\hat{c}_1(1-\gamma) + 1\right)\\
& \le \beta_T (\delta) \left(\hat{c}_1(1-\gamma) + 1\right) \sqrt{2Td \log (1+\frac{T}{\lambda})}.
\end{align*}
\end{proof}

\begin{proof}[Proof of Theorem \ref{thm:LinBUCB2}]
This result follows from a combination of the proof of Theorems \ref{thm:LinTS} and \ref{thm:LinBUCB1}. 

1) As shown in Theorem \ref{thm:LinTS}, for the approximate distribution $\tilde{\theta}_t \sim Q_t$, $\beta_t(\delta)^{-1} V_t^{1/2}(\tilde{\theta}_t-\hat{\theta}_t)$ satisfies similar anti-concentration and concentration properties in Assumption \ref{def_D1}, while the only difference is the constant terms. Hence, applying Theorem \ref{thm:LinBUCB1} Part 1) to $\tilde{\theta}_t \sim Q_t$, we obtain that

$$R(T) \le \beta_T (\delta) \left(\sqrt{c_2d\log \frac{c_2'd}{1-\gamma} } + 1\right) \sqrt{2Td \log (1+\frac{T}{\lambda})}.
$$

2) Note that $\tilde{\theta}_t \sim Q_t$. We define another random variable $\tilde{\theta}'_t \sim \Pi_t$ that follows the distribution $\Pi_t$. By Assumption \ref{assu0}, we have
$D_{\alpha_2}(\tilde{\theta}'_t,\tilde{\theta}_t)\le \epsilon$ where $\alpha_2<0$. By Assumption \ref{def_D2} Part 2), we know that for any $\zeta>0$, $\beta_t(\delta)^{-1} V_t^{1/2}(\tilde{\theta}'_t-\hat{\theta}_t)$ satisfies that 
for any $u \in \mathbb{R}^d $ with $\|u\| = 1$, 
$$\mathbb{P}(u^\top \beta_t(\delta)^{-1} V_t^{1/2}(\tilde{\theta}'_t-\hat{\theta}_t) \leq \hat{c}_1(\zeta)) \geq 1-\zeta.$$ 
By Theorem \ref{thm:concentration} Part 3), we have a similar concentration property for $\beta_t(\delta)^{-1} V_t^{1/2}(\tilde{\theta}_t-\hat{\theta}_t)$, with a different constant term: For any $u \in \mathbb{R}^d $ with $\|u\| = 1$, 
$$\mathbb{P}(u^\top \beta_t(\delta)^{-1} V_t^{1/2}(\tilde{\theta}_t-\hat{\theta}_t) \leq \hat{c}_2(\zeta)) \geq 1-\zeta,$$
where $\hat{c}_2(\zeta) = \hat{c}_1(\zeta^{\frac{\alpha_2-1}{\alpha_2}} \left(\epsilon\alpha_2(\alpha_2-1)+1\right)^{\alpha_2})$.

Therefore, for the approximate distribution $\tilde{\theta}_t \sim Q_t$, $\beta_t(\delta)^{-1} V_t^{1/2}(\tilde{\theta}_t-\hat{\theta}_t)$ satisfies similar anti-concentration and concentration properties in Assumption \ref{def_D2}, while the only difference is the constant terms. Hence, applying Theorem \ref{thm:LinBUCB1} Part 2) to $\tilde{\theta}_t \sim Q_t$, we conclude that
$$R(T) \le \beta_T (\delta) \left(\hat{c}_2(1-\gamma) + 1\right) \sqrt{2Td \log (1+\frac{T}{\lambda})}
$$
as desired.

\end{proof}

\subsection{Proofs of Results in Section \ref{sec:negativeresults}}

\begin{proof}[Proof of Theorem \ref{thm:tsfails}]
Suppose that $\mathcal{X}_t \equiv \{(1,0), (0,1))\}$, i.e., there are two arms throughout. 
Let $\pi_t(x_1, x_2)$ denote the density of $\Pi_t$.  
We construct a distribution $Q_t$ as follows:
\begin{gather}
q_{t}(x_1 ,x_2) = \begin{cases}
\frac{1-\frac{1}{r}(1-F_{t})}{F_{t}} \pi_{t}(x_1, x_2)  & \text{if } x_1 < x_2 \\
\frac{1}{r} \pi_{t}(x_1, x_2) & \text{if } x_1 \ge x_2,  \label{exp:tsfails}
\end{cases} 
\end{gather}
where $F_{t} = \mathbb{P}_{\Pi_t} (x_1 < x_2)$.

Step 1: We show that $q_{t}$ is indeed a density. $q_{t}\ge 0$ is obvious since $\pi_{t}\ge 0$. Let $\Omega_1 = \{(x_1, x_2): x_1 < x_2\}$ and $\Omega_2 = \{(x_1, x_2): x_1 \ge x_2\}$.
\begin{align*}
\int q_{t}(x_1, x_2)dx&=\int_{\Omega_1} q_{t}(x_1, x_2)dx+\int_{\Omega_2} q_{t}(x_1, x_2)dx\\
&=\int_{\Omega_1} \frac{1-\frac{1}{r}(1-F_{t})}{F_{t}} \pi_{t}(x_1, x_2) dx+\int_{\Omega_2} \frac{1}{r} \pi_{t}(x_1, x_2)dx\\
&=\frac{1-\frac{1}{r}(1-F_{t})}{F_{t}} F_{t} + \frac{1}{r}(1-F_{t})\\
&=1.
\end{align*}
Step 2: We show that there exists an $r>1$ independent of $t$ such that $q_{t}$ satisfies Assumption \ref{assu3}.

When $\alpha>1$ or $0< \alpha< 1$, we have:
\begin{align*}
 &D_{\alpha}(\Pi_{t}, Q_t)\\
=&\frac{1}{\alpha(\alpha-1)}\left(\int_{\Omega_1} q_{t}(x_1, x_2) \left(\frac{\pi_{t}(x_1, x_2)}{q_{t}(x_1, x_2)}\right)^{\alpha}dx +\int_{\Omega_2} q_{t}(x_1, x_2) \left(\frac{\pi_{t}(x_1, x_2)}{q_{t}(x_1, x_2)}\right)^{\alpha}dx -1\right)\\
=&\frac{1}{\alpha(\alpha-1)}\left(\int_{\Omega_1} \frac{1-\frac{1}{r}(1-F_{t})}{F_{t}} \pi_{t}(x_1, x_2) \left(\frac{F_{t}}{1-\frac{1}{r}(1-F_{t})} \right)^{\alpha}dx +\int_{\Omega_2} \frac{1}{r} \pi_{t}(x_1, x_2) r^{\alpha} dx -1\right)\\
= &\frac{1}{\alpha(\alpha-1)}\left( \left(\frac{1-\frac{1}{r}(1-F_{t})}{F_{t}}\right)^{1-\alpha} (F_{t})
+r^{\alpha-1} (1-F_{t})-1\right).
\end{align*}
Note that 
$$\frac{1-\frac{1}{r}(1-F_{t})}{F_{t}}\ge \frac{\frac{1}{r}-\frac{1}{r}(1-F_{t})}{F_{t}}=\frac{1}{r},$$
as $r>1$. Hence we have
$$D_{\alpha}(\Pi_{t}, Q_t)\le \frac{1}{\alpha(\alpha-1)}\left( \left(\frac{1}{r}\right)^{1-\alpha} (F_{t})
+r^{\alpha-1} (1-F_{t})-1\right)=\frac{1}{\alpha(\alpha-1)}\left(r^{\alpha-1} -1\right).$$
Then, for $1<r<(\epsilon\alpha(\alpha-1)+1)^{\frac{1}{\alpha-1}}$ (only if $\epsilon\alpha(\alpha-1)+1>0$, otherwise we put $+\infty$ as the upper bound on $r$), we have that
$$D_{\alpha}(\Pi_{t}, Q_t)\le\epsilon.$$

When $\alpha=1$:
\begin{align*}
D_{1}(\Pi_{t}, Q_{t})&=KL(\Pi_{t},Q_{t})\\
&=\int_{\Omega_1} \pi_{t}(x_1, x_2) \log\left(\frac{\pi_{t}(x_1, x_2)}{q_{t}(x_1, x_2)}\right)dx+\int_{\Omega_2} \pi_{t}(x_1, x_2) \log\left(\frac{\pi_{t}(x_1, x_2)}{q_{t}(x_1, x_2)}\right)dx\\
&=\int_{\Omega_1} \pi_{t}(x_1, x_2) \log\left(\frac{F_{t}}{1-\frac{1}{r}(1-F_{t})}\right) dx+\int_{\Omega_2} \pi_{t}(x_1, x_2)  \log\left(r\right)dx\\
&= \log\left(\frac{F_{t}}{1-\frac{1}{r}(1-F_{t})}\right) F_{t}+\log\left(r\right) (1-F_{t}).
\end{align*}
We note that 
$$\frac{F_{t}}{1-\frac{1}{r}(1-F_{t})}\le \frac{F_{t}}{\frac{1}{r}-\frac{1}{r}(1-F_{t})}=r,$$
as $r>1$. Hence we have
$$KL(\Pi_{t},Q_{t})\le \log\left(r\right) F_{t}+\log\left(r\right)(1-F_{t})=\log\left(r\right).$$
Then for $1<r<e^{\epsilon}$, we have that
$$D_{1}(\Pi_{t}, Q_{t})=KL(\Pi_{t},Q_{t})\le \log\left(r\right)\le \log\left(e^{\epsilon}\right)=\epsilon.$$

Step 3: We show that LinTS sampling from $Q_{t}$ incurs a linear regret. 

Since we have $\mathcal{X}_t \equiv \{(1,0), (0,1))\}$, by the definition of the regret, we have
$$R(T) =(\mu_1 - \mu_2) \sum_{t=1}^T \mathbb{E}\left[\mathbb{I}_{\text{choose } (0,1)}(t)\right],$$
where $\mathbb{I}_{\text{choose } (0,1)}(t)$ represents the indicator of LinTS choosing the context $(0,1)$ at time step $t$.

Since $\mathbb{P}_{Q_t} (x_1 \ge x_2) \le \frac{1}{r}$ and $\mathbb{P}_{Q_t} (x_1 < x_2) \ge 1-\frac{1}{r}$. This shows that with probability at least $1-\frac{1}{r}$, LinTS with $Q_t$ chooses the context $(0,1)$ at any time step $t$. Hence, we conclude that $\mathbb{E}\left[\mathbb{I}_{\text{choose} (0,1)}(t)\right]\ge 1-\frac{1}{r}$ and the lower bound of the expected regret is given by
$$R(T) \ge (\mu_1-\mu_2)(1-\frac{1}{r})T=\Omega(T),$$
leading to a linear regret.
\end{proof}

\begin{proof}[Proof of Theorem \ref{thm:bucbfails}]
Suppose that $\mathcal{X}_t \equiv \{(1,0), (0,1))\}$, i.e., there are two arms throughout. 
Let $\pi_t(x_1, x_2)$ denote the density of $\Pi_t$.  
We construct a distribution $Q_t$ as follows:

We set $b_t=\mathcal{Q}(\gamma,\Pi_{t,1})\in (0,1)$, the $\gamma$-quantile of $\Pi_{t,1}$ (the marginal distribution of $\Pi_{t}$ on $x_1$). Let $F^{2|1}_{t} = \mathbb{P}_{\Pi_t}(x_2 \le b_t | x_1) $ be the conditional cumulative distribution at $b_t$ of $x_2$ given $x_1$. For $r>1$, we set
\begin{equation*}
q_{t}(x_1, x_2) = \begin{cases}
\frac{1}{r} \pi_{t}(x_1, x_2) & \text{if } x_2 < b_t \\
\frac{1-\frac{1}{r}F^{2|1}_t}{1-F^{2|1}_t} \pi_{t}(x_1, x_2)  & \text{if } x_2 \ge b_t.
       \end{cases} \quad
\end{equation*}
Step 1: We show that $q_{t}(x_1, x_2)$ is a density, and $q_{t}(x_1) = \pi_{t}(x_1)$, i.e., the marginal distributions of $\Pi_t$ and $Q_t$ on $x_1$ are identical. First, $q_{t}\ge 0$ is obvious since $\pi_{t}\ge 0$. Second, 
\begin{align*}
q_{t}(x_1)&= \int q_{t}(x_1, x_2)dx_2\\
&=\int_{-\infty}^{b_t} q_{t}(x_1, x_2)dx_2+\int_{b_t}^{\infty} q_{t}(x_1, x_2)dx_2\\
&=\int_{-\infty}^{b_t} \frac{1}{r} \pi_{t}(x_1, x_2)dx_2+\int_{b_t}^{\infty} \frac{1-\frac{1}{r}F^{2|1}_t}{1-F^{2|1}_t} \pi_{t}(x_1, x_2)dx_2\\
&= \int_{-\infty}^{b_t} \frac{1}{r} \pi_{t}(x_1, x_2)dx_2 dx_1 +\int_{b_t}^{\infty} \frac{1-\frac{1}{r}F^{2|1}_t}{1-F^{2|1}_t} \pi_{t}(x_1, x_2)dx_2 dx_1\\
&= \frac{1}{r} F^{2|1}_t \pi_t(x_1) + \frac{1-\frac{1}{r}F^{2|1}_t}{1-F^{2|1}_t} (1-F^{2|1}_t) \pi_t(x_1)\\
&=\pi_t(x_1).
\end{align*}
This shows that the marginal distributions of $\Pi_t$ and $Q_t$ on $x_1$ are identical. We also have
\begin{align*}
\int q_{t}(x_1, x_2)dx&= \int_{-\infty}^{\infty} q_{t}(x_1)dx_1\\
&=\int_{-\infty}^{\infty} \pi_{t}(x_1)dx_1\\
&=1.
\end{align*}

Step 2: We show that there exists an $r> \frac{1}{\gamma}$ independent of $t$ such that $q_{t}$ satisfies Assumption \ref{assu3}.

When $\alpha>1$ or $0< \alpha< 1$, we have:
\begin{align*}
&D_{\alpha}(\Pi_{t}, Q_t)\\ 
=&\frac{1}{\alpha(\alpha-1)}\left(\int_{\Omega_1} q_{t}(x_1, x_2) \left(\frac{\pi_{t}(x_1, x_2)}{q_{t}(x_1, x_2)}\right)^{\alpha}dx +\int_{\Omega_2} q_{t}(x_1, x_2) \left(\frac{\pi_{t}(x_1, x_2)}{q_{t}(x_1, x_2)}\right)^{\alpha}dx -1\right)\\
=&\frac{1}{\alpha(\alpha-1)}\left(\int_{-\infty}^{\infty} \int_{-\infty}^{b_t} \pi_{t}(x_1, x_2) r^{\alpha-1} dx_2 dx_1 +\int_{-\infty}^{\infty} \int_{b_t}^{\infty} \left(\frac{1-\frac{1}{r}F^{2|1}_t}{1-F^{2|1}_t}\right)^{1-\alpha} \pi_{t}(x_1, x_2)dx_2 dx_1 -1 \right)\\
=&\frac{1}{\alpha(\alpha-1)}\left(\int_{-\infty}^{\infty}  F^{2|1}_t r^{\alpha-1} \pi_t(x_1) dx_1 +\int_{-\infty}^{\infty} \left(\frac{1-\frac{1}{r}F^{2|1}_t}{1-F^{2|1}_t}\right)^{1-\alpha} (1-F^{2|1}_t) \pi_t(x_1) dx_1 -1 \right)\\
= &\frac{1}{\alpha(\alpha-1)} \int_{-\infty}^{\infty}\left( \left(\frac{1-\frac{1}{r}F^{2|1}_t}{1-F^{2|1}_t}\right)^{1-\alpha} (1-F^{2|1}_t)
+r^{\alpha-1} F^{2|1}_t-1\right) \pi_t(x_1) dx_1.
\end{align*}
Note that 
$$\frac{1-\frac{1}{r}F^{2|1}_t}{1-F^{2|1}_t}\ge \frac{\frac{1}{r}-\frac{1}{r}F^{2|1}_t}{1-F^{2|1}_t}=\frac{1}{r},$$
as $r>1$. Hence we have
$$D_{\alpha}(\Pi_{t}, Q_t)\le \int_{-\infty}^{\infty}\left( \frac{1}{\alpha(\alpha-1)}\left(r^{\alpha-1} (1-F^{2|1}_t)
+r^{\alpha-1} F^{2|1}_t-1\right) \right) \pi_t(x_1) dx_1=\frac{1}{\alpha(\alpha-1)}\left(r^{\alpha-1} -1\right).$$
Then for $\frac{1}{\gamma} <r \le (\epsilon\alpha(\alpha-1)+1)^{\frac{1}{\alpha-1}}$ (only if $\epsilon\alpha(\alpha-1)+1>0$, otherwise we put $+\infty$ as the upper bound on $r$), we have that
$$D_{\alpha}(\Pi_{t}, Q_t)\le\epsilon.$$

When $\alpha=1$:
\begin{align*}
&D_{1}(\Pi_{t}, Q_{t})\\
=&KL(\Pi_{t},Q_{t})\\
=&\int_{\Omega_1} \pi_{t}(x_1, x_2) \log\left(\frac{\pi_{t}(x_1, x_2)}{q_{t}(x_1, x_2)}\right)dx+\int_{\Omega_2} \pi_{t}(x_1, x_2) \log\left(\frac{\pi_{t}(x_1, x_2)}{q_{t}(x_1, x_2)}\right)dx\\
=&\int_{-\infty}^{\infty} \int_{-\infty}^{b_t} \pi_{t}(x_1, x_2) \log\left(r\right) dx_2 dx_1+\int_{-\infty}^{\infty} \int_{-\infty}^{b_t} \pi_{t}(x_1, x_2) \log\left(\frac{1-F^{2|1}_t}{1-\frac{1}{r}F^{2|1}_t}\right) dx_2 dx_1\\
=& \int_{-\infty}^{\infty} \left( \log\left(\frac{1-F^{2|1}_t}{1-\frac{1}{r}F^{2|1}_t}\right) (1-F^{2|1}_t)+\log\left(r\right) F^{2|1}_t  \right) dx_1.
\end{align*}
Note that 
$$\frac{1-F^{2|1}_t}{1-\frac{1}{r}F^{2|1}_t}\le \frac{1-F^{2|1}_t}{\frac{1}{r}-\frac{1}{r}F^{2|1}_t}=r,$$
as $r>1$. Hence we have
$$KL(\Pi_{t},Q_{t})\le  \int_{-\infty}^{\infty} \left( \log\left(r\right) (1-F^{2|1}_t)+\log\left(r\right)F^{2|1}_t \right) dx_1=\log\left(r\right).$$
Then for $\frac{1}{\gamma}<r\le e^{\epsilon}$, we have that
$$D_{1}(\Pi_{t}, Q_{t})=KL(\Pi_{t},Q_{t})\le \log\left(r\right)\le \log\left(e^{\epsilon}\right)=\epsilon.$$

Step 3: We show that LinBUCB sampling from $Q_{t}$ incurs a linear regret. 

Since we have $\mathcal{X}_t \equiv \{(1,0), (0,1))\}$, by the definition of the regret, we have
$$R(T) =(\mu_1 - \mu_2) \sum_{t=1}^T \mathbb{E}\left[\mathbb{I}_{\text{choose } (0,1)}(t)\right],$$
where $\mathbb{I}_{\text{choose } (0,1)}(t)$ represents the indicator of LinBUCB choosing the context $(0,1)$ at time step $t$.

Note that 
$$\mathbb{P}_{Q_{t}}(x_1\le b_{t})=\mathbb{P}_{\Pi_{t}}(x_1\le b_{t})=\gamma,$$ 
as in Step 1, we have shown that the marginal distributions of $\Pi_t$ and $Q_t$ on $x_1$ are identical. 
$$\mathbb{P}_{Q_{t}}(x_2\le b_{t})= \int_{-\infty}^{\infty} \int_{-\infty}^{b_t} \frac{1}{r} \pi_{t}(x_1, x_2)dx_2dx_1\le \frac{1}{r} < \gamma$$
by taking $r = (\epsilon\alpha(\alpha-1)+1)^{\frac{1}{\alpha-1}}$ or $r=e^{\epsilon}$. This implies that 
$$\mathcal{Q}(\gamma,Q_{t,1})=b_{t}< \mathcal{Q}(\gamma,Q_{t,2}).$$
This shows that LinBUCB will always choose the context $(0,1)$ at any time step $t$. Hence, we conclude that $\mathbb{E}\left[\mathbb{I}_{\text{choose} (0,1)}(t)\right] = 1$ and the lower bound of the expected regret is given by
$$R(T) \ge (\mu_1-\mu_2)T=\Omega(T),$$
leading to a linear regret.
\end{proof}


\section{Pseudocodes} \label{sec:pseudocode}
In this section, we present the pseudocode for LinTS and LinBUCB under both exact and approximate inference settings.

\subsection{LinTS w/o Approximate Inference}  \label{code_lints}
The pseudocode of LinTS w/o approximate inference is presented in Algorithm \ref{algo:LinTS}.

\subsection{LinBUCB w/o Approximate Inference} \label{code_linbucb}

The pseudocode of LinBUCB w/o approximate inference is presented in Algorithm \ref{algo:LinBUCB}.

\begin{algorithm}[ht]
\caption{LinTS w/o Approximate Inference}
\label{algo:LinTS}
\begin{algorithmic}[1]
\Require{ $\hat{\theta}_1$ and $V_1$ given by Equation \eqref{RLS_estimation}, probability threshold $\delta$, and time horizon $T$}
\State Set $\delta' = \delta/(4T)$
\For{ $t = {1, ..., T}$ } 
\State Draw a posterior sample $\tilde{\theta}_t$ as follows:

 \If{``exact inference''} 

\State Draw $\tilde{\theta}_t\sim \Pi_t$, where 
$\beta_t(\delta')^{-1} V_t^{1/2}(\tilde{\theta}_t-\hat{\theta}_t)$ satisfies Assumption \ref{def_D1}

 \Else \ \ \ \textcolor{gray}{\# if ``approximate inference''} 

\State Draw $\tilde{\theta}_t\sim Q_t$, 
where $Q_t$ is an approximation of $\Pi_t$
\EndIf 

\State Compute optimal arm $x_t = x_t^*(\tilde \theta_t) = \argmax{x\in \mathcal{X}_t} x^\top \tilde{\theta}_t$ 

\State Pull arm $x_t$ and observe reward $r_{t+1}$

\State  Compute $V_{t+1}$ and $\hat{\theta}_{t+1}$ using Equation \eqref{RLS_estimation}
\EndFor
\end{algorithmic}

\end{algorithm}

\begin{algorithm}[ht]
\caption{LinBUCB w/o Approximate Inference}
\label{algo:LinBUCB}
\begin{algorithmic}[1]
\Require{ $\hat{\theta}_1$ and $V_1$ given by Equation \eqref{RLS_estimation}, probability threshold $\delta$, time horizon $T$, and quantile level $\gamma$ (specified in Theorems \ref{thm:LinBUCB1}, \ref{thm:LinBUCB2})}
\For{$t = {1, ..., T}$}
\State Compute a posterior quantile for each arm as follows:

\If{``exact inference''}
 
\State Suppose $\tilde{\theta}_t\sim \Pi_t$, 
where 
$\beta_t(\delta)^{-1} V_t^{1/2}(\tilde{\theta}_t-\hat{\theta}_t)$ satisfies Assumption \ref{def_D1} or Assumption \ref{def_D2}. Then for each $x\in\mathcal{X}_t$, compute a $\gamma$-quantile of the distribution of $x^\top\tilde{\theta}_t$, i.e., $\mathcal{Q}(\gamma, x^\top\tilde{\theta}_t)$

\Else  \ \ \ \  \textcolor{gray}{ \# if ``approximate inference''}
\State Suppose  $\tilde{\theta}_t\sim Q_t$,
where $Q_t$ is an approximation of $\Pi_t$. Then compute $\mathcal{Q}(\gamma, x^\top\tilde{\theta}_t)$
\EndIf 
\State Compute optimal arm $x_t = \argmax{x\in \mathcal{X}_t} \mathcal{Q}(\gamma, x^\top\tilde{\theta}_t)$ 
\State Pull arm $x_t$ and observe reward $r_{t+1}$
\State Compute $V_{t+1}$ and $\hat{\theta}_{t+1}$ using Equation \eqref{RLS_estimation}
 \EndFor
\end{algorithmic}

\end{algorithm}

\section{Additions to Experiments}\label{sec:expmore}
In this section, we present supplementary experimental results and outline the implementation details and analyses of the proposed study.

\subsection{Computational Comparison of Different Methods} \label{sec:comput_method}
We conduct the runtime experiments to study the computational efficiency of the proposed approximate method. The results, presented in Tables \ref{comp:ts} and \ref{comp:bucb}, show that the proposed approximate method can significantly reduce computational time, demonstrating a notable improvement in efficiency. These findings further highlight the practical advantages of the proposed approximate approaches.

\begin{table}
\centering
\caption{Computational time of LinTS and LinTS\_Approximate on $P3$ ($K = 10$)  with different context dimensions: $d \in [500, 2000]$.}
\label{comp:ts}
\scalebox{1.0}{
\begin{tblr}{
colspec = {c c c c c c c c }, 
rowsep = 1pt,
  hline{1,2,5} = {-}{}, 
}
$d$	&500	&750	&1000	&1250	&1500	&1750	&2000 \\
LinTS (s)	& 1.33	& 3.65 &	6.27 &	10.61 &	15.81 &	22.99 &	31.05\\
LinTS\_Aprx (s)	&0.76	&2.40&	3.84&	6.23&	9.43&	13.17&	17.60\\
Speed-up Ratio	&1.75	&1.52	&1.63 &	1.70	& 1.67&	1.74&	1.76
\\

\end{tblr}}
\end{table}

\begin{table}
\centering
\caption{Computational time of LinBUCB and LinBUCB\_Approximate on $P3$ ($K = 10$) with different context dimensions: $d \in [500, 2000]$.}
\label{comp:bucb}
\scalebox{1.0}{
\begin{tblr}{
colspec = {c c c c c c c c }, 
rowsep = 1pt,
  hline{1,2,5} = {-}{}, 
}
$d$	&500	&750	&1000	&1250	&1500	&1750	&2000 \\
LinBUCB (s)	&1.81& 4.23 & 6.87 & 11.84 & 17.24 & 24.07 & 33.09 \\
LinBUCB\_Aprx (s)	& 1.28	& 2.65	& 3.96 & 6.44 & 9.03 & 12.48&16.86 \\
Speed-up Ratio	& 1.41& 1.60& 1.73& 1.84& 1.91& 1.93& 1.96\\

\end{tblr}}
\end{table}

\subsection{Computational Comparison of Matrix Inversion Approaches} \label{sec:comput_matrix}
In Table \ref{Computation}, we compare the execution time of exact matrix inversion versus the proposed diagonal approximation inversion across various dimensions. These results indicate that the proposed approximate method can potentially reduce computational load in real-world applications while maintaining comparable performance.

\begin{table}
\centering
\caption{Computational time of exact matrix inversion and diagonal approximation matrix inversion with different dimensions $d \in [2000, 30000]$.}
\label{Computation}
\scalebox{1.0}{
\begin{tblr}{
colspec = {c c c c c c c }, 
rowsep = 1pt,
  hline{1,2,5} = {-}{}, 
}
$d$	&	2000 &	5000	&10000&	15000&	20000&	25000	&30000 \\
$V^{-1} (regular) (s)$	&0.0499	& 0.4296	& 2.9354&	9.6358&	21.5713	&44.2227	&79.7421 \\
$D^{-1} (approximate) (s)$	& 0.0039&	0.0097&	0.0541	&0.1480&	0.2646&	0.3986	&0.5496 \\
Speed-up Ratio	& 12.93	&44.10	&54.26	&65.07	&81.52	&110.96	&145.07
\\

\end{tblr}}
\end{table}

\subsection{Scalability Analysis} \label{sec:results}

In this section, we conduct extensive experiments on Problem Setting $P3$ (where the ground-truth parameter vector $\theta^*$ is a sample generated once from $\text{Uniform}(0, 1)^{\otimes d}$) to show the computational scalability of the proposed methods. Specifically, we test a wide range of settings, including higher dimensions ($d \in {5, 10, 20, 30, 40, 50}$) and larger arm sizes ($K \in {5, 10, 20, 30, 40, 50}$), and presented the results in Figure \ref{fig:s1} and Figure \ref{fig:s2}, correspondingly. These experiments show that our methodology and findings remain robust across diverse parameter scales and distributions.

\begin{figure*}[ht]
    \centering
    \includegraphics[width=0.99\textwidth]{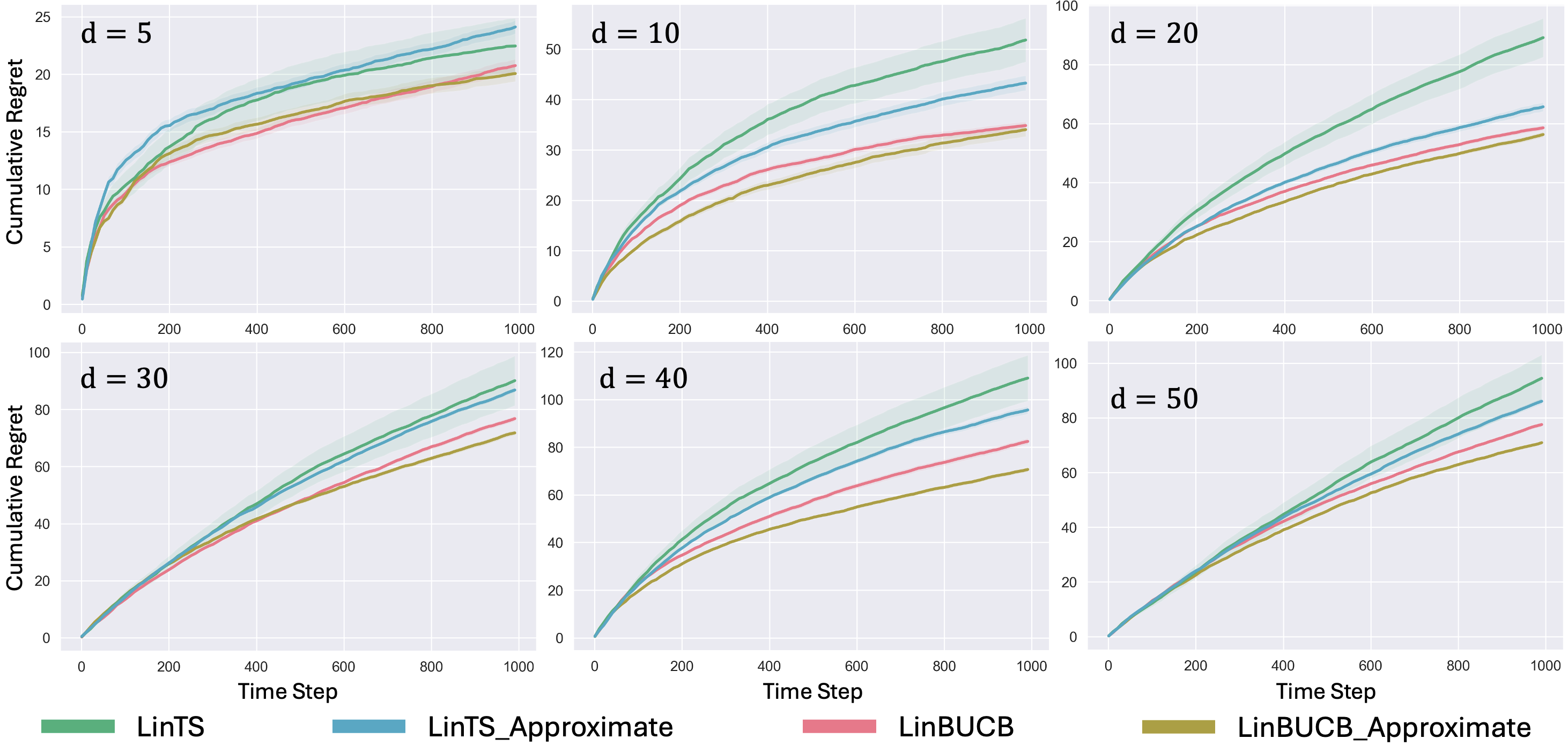}
    \caption{Results of LinBUCB, LinTS, LinBUCB$\_$Approximate, and LinTS$\_$Approximate under different feature dimensions $d$ on Problem Setting $P3$ with the fixed number of arms $K = 10$. Results are averaged over 10 runs with shaded standard errors.} 
    \label{fig:s1}
\end{figure*}

\begin{figure*}[ht]
    \centering
    \includegraphics[width=0.99\textwidth]{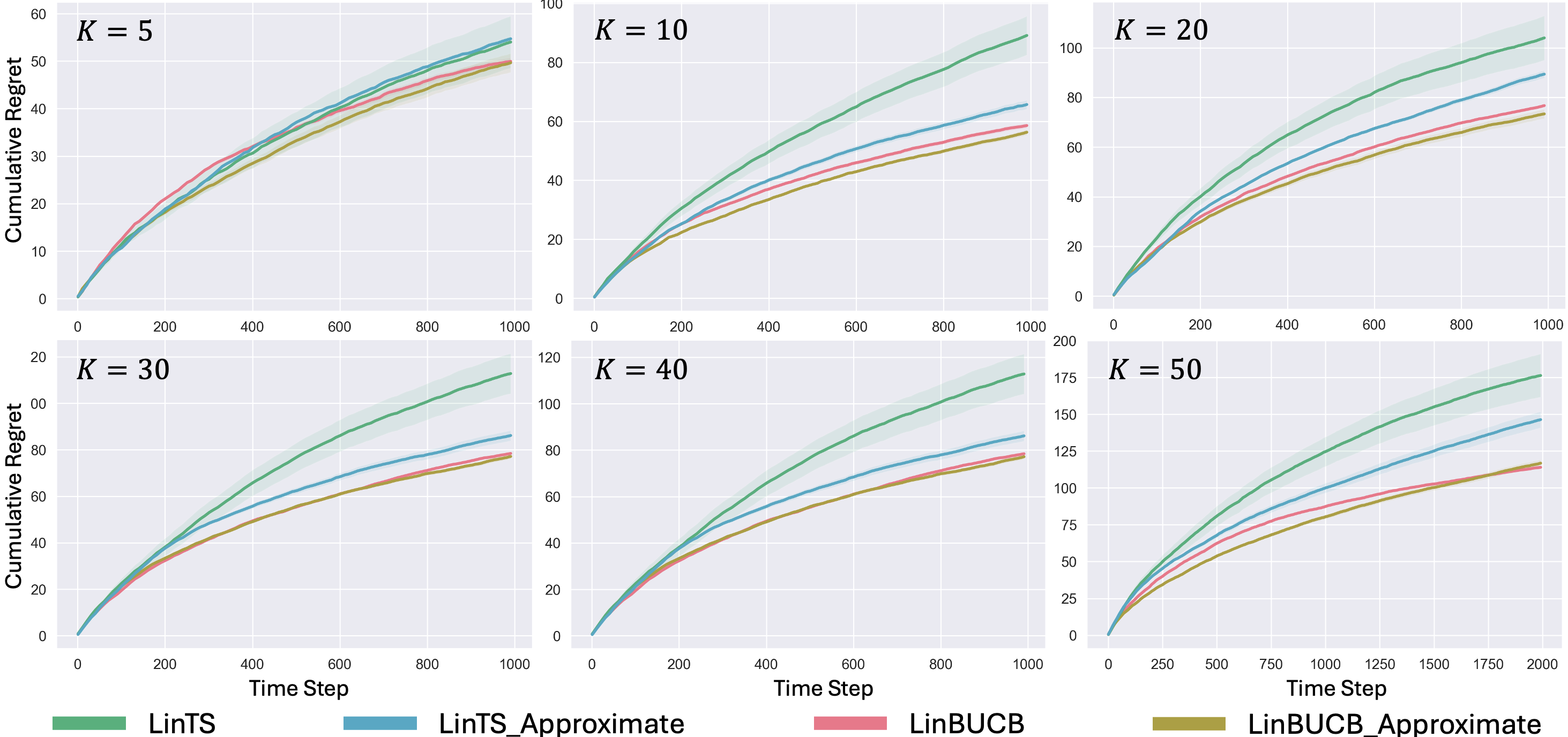}
    \caption{Results of LinBUCB, LinTS, LinBUCB$\_$Approximate, and LinTS$\_$Approximate under different number of arms $K$ on Problem Setting $P3$ with the fixed feature dimension $d = 20$. Results are averaged over 10 runs with shaded standard errors.} 
    \label{fig:s2}
\end{figure*}

\subsection{Sensitivity Analysis} \label{sec:sensitivity}

We conduct the sensitivity analysis to examine the impact of different quantile $\gamma$ in LinBUCB and LinBUCB$\_$Approximate, and present the results in Figure \ref{fig:s3}. LinBUCB\_Approximate and LinBUCB achieve optimal performance at around $\gamma = 0.6$, and maintain robust performance within the range $\gamma \in [0.5, 0.7]$. These findings align with our theoretical results. As established in Theorems \ref{thm:LinBUCB1} and \ref{thm:LinBUCB2}, a relatively large value of $\gamma$ generally leads to a favorable regret rate. However, an excessively large $\gamma$ may introduce over-conservativeness, resulting in higher constants in the regret bound. 

\begin{figure*}[ht]
    \centering
    \includegraphics[width=0.99\textwidth]{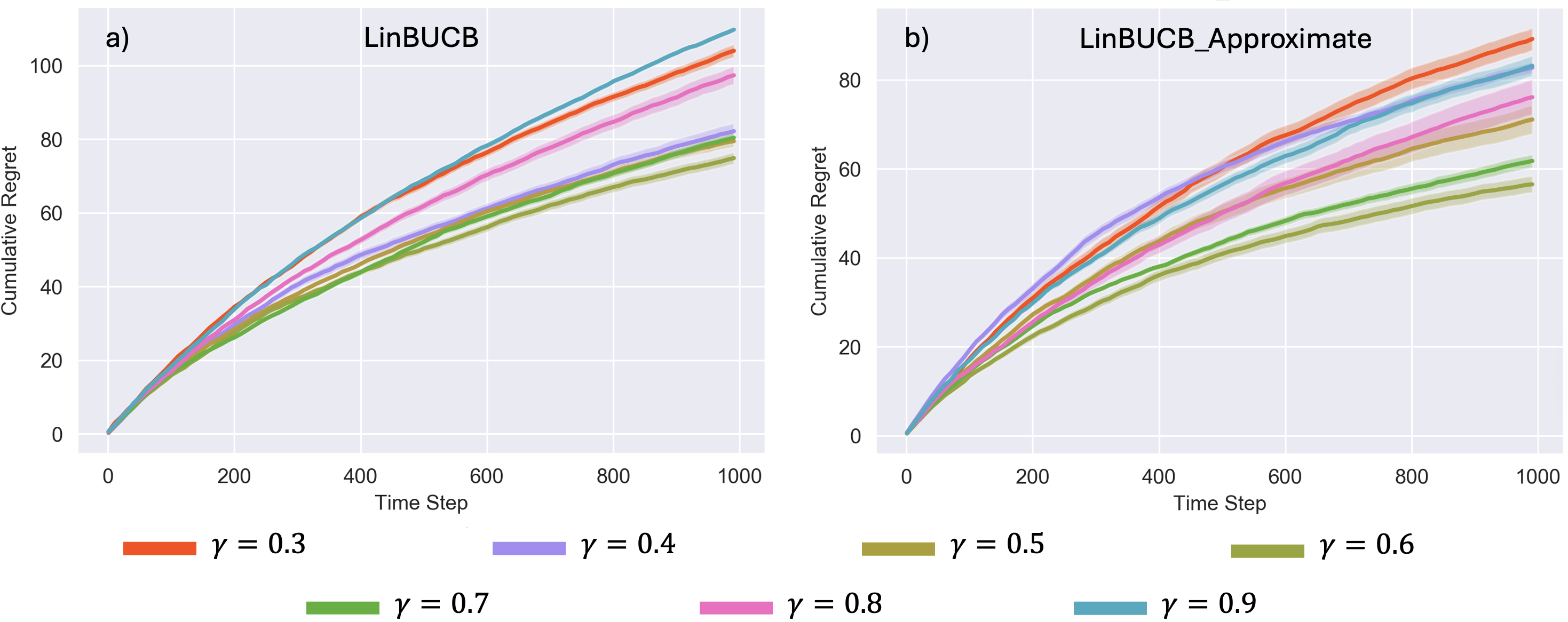}
    \caption{Sensitivity analysis of different quantile $\gamma$ in LinBUCB (a) and LinBUCB$\_$Approximate (b) on Problem Setting $P3$ ($d = 20$, $K = 10$, $T = 1000$). Results are averaged over 5 runs with shaded standard errors. } 
    \label{fig:s3}
\end{figure*}

\subsection{Implementation Details} \label{sec:implement}
All experiments are conducted on a single CPU (9th Gen Intel CoreTM i9 9900K). The execution time for a single experiment is reported in Tables \ref{comp:ts} and \ref{comp:bucb}. In the data generation process, independent noise $\xi_{t} \sim \mathcal{N}(0, 0.5^2)$ is added to the random reward. In LinTS and LinBUCB, we set the regularization parameter $\lambda=1$ and use the prior $\mathcal{N}(0,I_d)$ that satisfies both Assumptions \ref{def_D1} and \ref{def_D2}. 

To implement LinTS$\_$Approximate and LinBUCB$\_$Approximate, we apply a fast approximate computing approach to mimic approximate inference. As shown in Algorithms \ref{algo:LinTS} and \ref{algo:LinBUCB}, the original two algorithms require computing the inverse $V_t^{-1}$ and drawing samples from this covariance matrix, which could be problematic in a high-dimensional setting. As such, instead of computing the inverse matrix exactly, we take the diagonal elements of $V_t$ as a diagonal matrix, $D_t=\text{diag}(V_t)$, and then compute its inverse, $D_t^{-1}$, as an approximation of $V_t^{-1}$. Then we replace $V_t^{-1}$ with $D_t^{-1}$ when implementing LinTS and LinBUCB.


\end{document}